%% file: main.tex
\documentclass[11pt,a4paper,logo,nonumbering]{map}

\usepackage{placeins}
\usepackage{tcolorbox}
\usepackage{algorithm}
\usepackage{algorithmic}
\usepackage{etoolbox}

\definecolor{policyrow}{HTML}{FCE8E6}
\definecolor{criticrow}{HTML}{E6F0FA}
\definecolor{maerow}{HTML}{E8E8E8}
\newtheorem{theorem}{Theorem}
\newtheorem{proposition}{Proposition}
\newtheorem{corollary}{Corollary}
\newtheorem{lemma}{Lemma}

\newcommand{\experimentplot}[2]{%
  \IfFileExists{#1}{\includegraphics[width=\linewidth]{#1}}{%
    \fbox{\parbox[c][2.6cm][c]{\dimexpr\linewidth-2\fboxsep-2\fboxrule\relax}{%
      \centering\small #2\par\medskip\textit{Results pending.}}}}}

\hypersetup{
  pdftitle={Multiscale Policy Optimization for Long-Horizon LLM Agents},
  pdfauthor={Tao Ren et al.},
  colorlinks=true,linkcolor=blue,citecolor=purple,urlcolor=blue
}

\title{\centering Multiscale Policy Optimization for Long-Horizon LLM Agents}
\makeatletter
\renewcommand{\@author}{%
  \centering\normalfont\fontsize{10}{13}\selectfont
  \mbox{Tao Ren\textsuperscript{1,*}},
  \mbox{Hui Yang\textsuperscript{1,*}},
  \mbox{Jia Deng\textsuperscript{2,*}},
  \mbox{Tingyang Chen\textsuperscript{2,*}},
  \mbox{Qiyao Wang\textsuperscript{2,*}},
  \mbox{Shuo Tang\textsuperscript{2,*}},\\
  \mbox{Ziyang Zeng\textsuperscript{2,*}},
  \mbox{Yuhua Jiang\textsuperscript{3}},
  \mbox{Guanghao Li\textsuperscript{3}},
  \mbox{Huatong Song\textsuperscript{4}},
  \mbox{Yining Li\textsuperscript{3}},
  \mbox{Yifan Luo\textsuperscript{2}},
  \mbox{Kunbao Ding\textsuperscript{2}},\\
  \mbox{Feng Chang\textsuperscript{2,\ensuremath{\dagger}}},
  \mbox{Chuan Hao\textsuperscript{2}},
  \mbox{Ran Tao\textsuperscript{2}},
  \mbox{Bryan Dai\textsuperscript{2}},
  \mbox{Jiafeng Liang\textsuperscript{6,\ensuremath{\dagger}}},
  \mbox{Jinyang Jiang\textsuperscript{5,\ensuremath{\dagger}}}\par\vspace{0.5em}
  {\small
  \textsuperscript{1}Peking University\quad
  \textsuperscript{2}IQuest Research\quad
  \textsuperscript{3}Tsinghua University\\
  \textsuperscript{4}Renmin University of China\quad
  \textsuperscript{5}The University of Hong Kong\quad
  \textsuperscript{6}Independent Researcher\par}
  \vspace{0.4em}
  {\small
  \textsuperscript{*}Equal contribution.\quad
  \textsuperscript{\ensuremath{\dagger}}Corresponding authors.\\
  Contact: \texttt{rtkenny@stu.pku.edu.cn}\par}%
}
\patchcmd{\maketitle}{\vskip20pt}{\vskip12pt}{}{}
\patchcmd{\maketitle}{\vskip30pt}{\vskip10pt}{}{}
\patchcmd{\abscontent}{\vspace{5ex}}{\vspace{1.5ex}}{}{}
\makeatother

\input{sections/abstract}

\begin{document}
\maketitle

\par\medskip
\noindent\begin{minipage}{\linewidth}
  \captionsetup{type=figure,font=small,hypcap=false}
  \centering
  \includegraphics[width=0.95\linewidth]{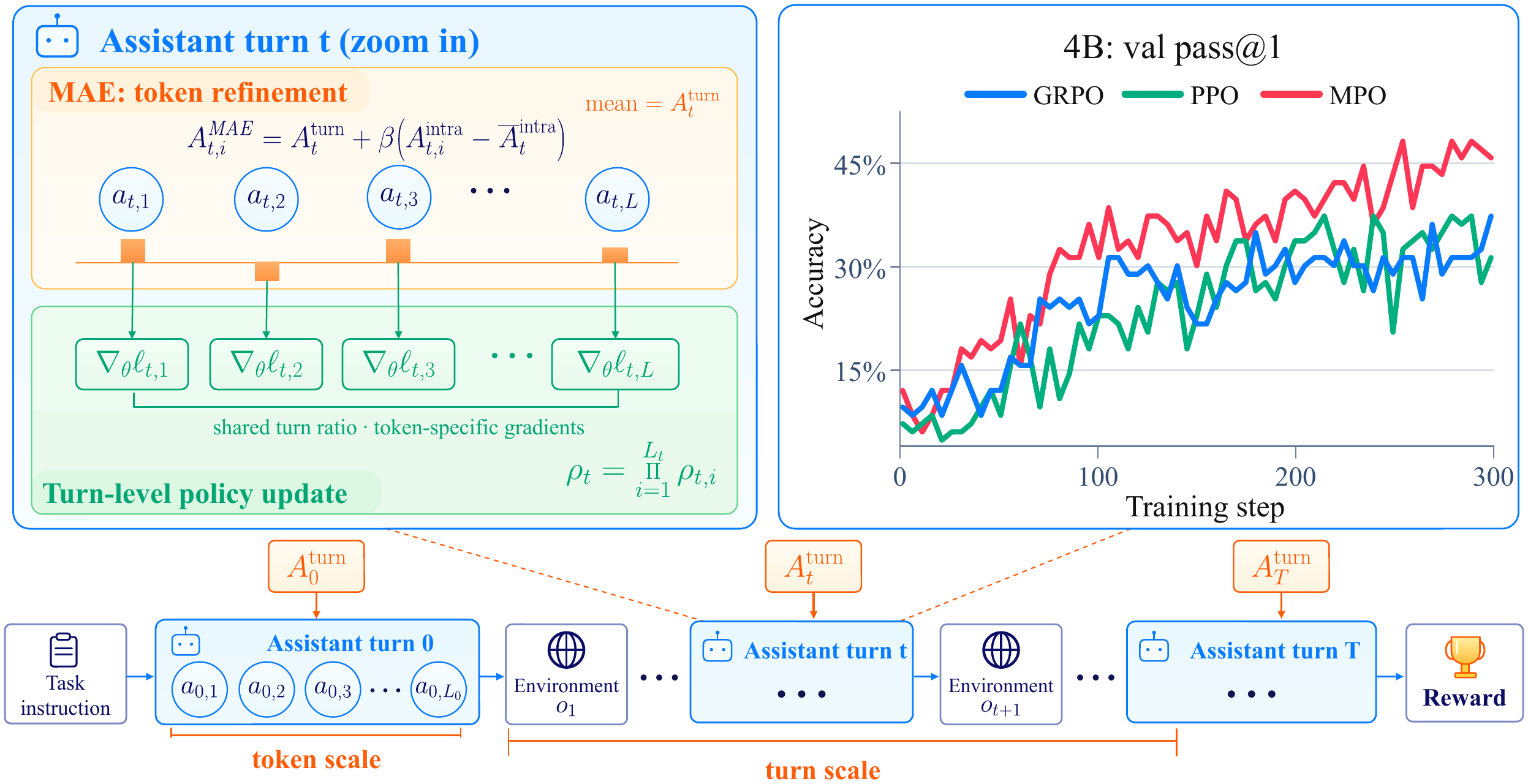}
  \caption{\textbf{Overview of Multiscale Policy Optimization (MPO).}
    Agents interact through complete assistant turns while generating actions
    token by token. MAE refines shared turn advantages into token-level credit
    while preserving each turn's mean. The policy update combines a shared
    complete-turn importance ratio with token-specific gradients.
    Top right: validation pass@1 during 4B training.}
  \label{fig:mpo-overview}
\end{minipage}
\clearpage

\input{sections/mpo}

\end{document}

%% file: sections/abstract.tex
\begin{abstract}
Long-horizon LLM agents interact with their environments through complete
assistant turns while generating each action token by token.
We observe that critic values vary primarily across turns, motivating
optimization anchored at turns
with finer credit allocation among tokens. We propose Multiscale Policy
Optimization (MPO), combining Multiscale Advantage Estimation (MAE) and
turn-level importance sampling to align credit assignment and policy
updates at the turn scale. Extensive experiments across search, tool-use,
and automation tasks with 4B and 30B-A3B models demonstrate consistent
gains over strong RL baselines.
Code is available at \url{https://github.com/IQuestLab/Multiscale-Policy-Optimization}.
\end{abstract}

%% file: sections/mpo.tex
\section{Introduction}
\label{sec:introduction}

An LLM agent interacts with its environment through complete assistant
turns, such as search queries, tool calls, and final responses, while
its policy generates each turn token by token. Learning from a terminal
task reward therefore involves two scale choices. Credit must connect
the final outcome to interactions across the trajectory and distinguish
the tokens within each interaction. Policy updates must also account
for changes in complete-output probabilities while retaining those
local distinctions. These choices motivate a policy-optimization
method that uses both the interaction and generation scales.

Two common families of RL methods leave a gap between these scales.
\textbf{Critic-free, group-relative methods}, including GRPO, DAPO,
and GSPO~\citep{shao2024deepseekmath,yu2025dapo,zheng2025gspo}, derive
one shared advantage per generated sequence from group-relative outcome
rewards. Despite differences in their update objectives, this shared
credit signal does not distinguish the contributions of individual
turns or tokens within a rollout. \textbf{Critic-based, token-level
methods}, including PPO, VC-PPO, VAPO, and
SAO~\citep{schulman2017ppo,yuan2025vcppo,yue2025vapo,hou2026sao}, use a learned
value function to estimate advantages along the token sequence.
Their common formulation treats each generated token as an action,
providing finer credit without explicitly separating within-turn
generation from turn-level environment interaction. Agentic RL therefore
calls for credit assignment that captures both the contribution of each
complete interaction and the relative contributions of tokens within it.

\begin{figure}[t]
 \centering
 \includegraphics[width=0.55\linewidth]{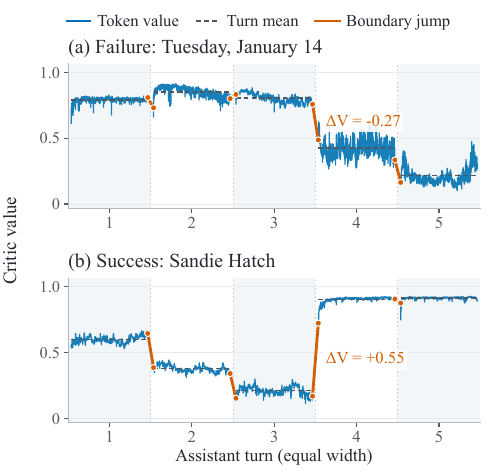}
 \caption{\textbf{Critic values exhibit structure at the turn scale.}
 A failed trajectory (a) and a successful trajectory (b) both exhibit
 relatively stable values within turns and pronounced changes across
 turn boundaries.}
 \label{fig:critic-turn-motivation}
\end{figure}

\paragraph{Observing credit structure across turns.}
Figure~\ref{fig:critic-turn-motivation} illustrates learned values that
vary within turns and exhibit larger changes across some environment
interactions. Across 2,048 trajectories and 16,323 turns from eight
checkpoint batches of one search-agent run, differences between turn
means account for a median of \textbf{89.5\%} of the critic's
token-weighted within-trajectory value variance. Between-turn variance exceeds
within-turn variance in \textbf{93.8\%} of trajectories. This structure
motivates propagating long-range credit across turns while retaining
local token information. Appendix~\ref{app:critic-observation} provides
additional statistics and detailed analyses of critic behavior in
multi-turn agent interactions.

\paragraph{Multiscale Policy Optimization.}
Treating complete assistant turns as macro actions, we establish a
turn-level improvement bound and show that token surrogates incur an
additional within-turn distribution-shift penalty for exact advantages
under the same KL constraints. Guided by this analysis and the observed
critic structure, we develop Multiscale Policy Optimization (MPO),
combining multiscale credit assignment with turn-level importance
weighting. Multiscale Advantage Estimation (MAE) uses turn-level GAE to
propagate credit across interactions, then refines token credit through
a centered local correction that preserves the turn mean. We derive
this allocation as a constrained projection using a single critic.
Turn-level importance weighting applies the complete-turn probability
ratio while retaining each token's MAE advantage in its policy gradient
through a stop-gradient construction.

Our contributions are threefold:
\begin{itemize}
 \item We analyze policy improvement for LLM agents at turn and token
 scales, deriving a turn-level improvement bound and identifying the
 additional discrepancy that token-level surrogates introduce through
 within-turn distribution shift.
 \item We propose MPO, which combines multiscale advantage estimation
 (MAE) with turn-level importance sampling to align credit assignment
 and policy updates with complete turns while preserving token-specific
 credit.
 \item We demonstrate consistent performance gains over strong RL
 baselines across search, tool-use, and automation tasks with both
 4B and 30B-A3B models.
\end{itemize}

\section{Related Work}
\label{sec:related-work}

\subsection{Tool use and multi-turn agent learning}
\label{sec:related-agents}

Early work connected language-model reasoning with actions and external
tools \citep{yao2023react,schick2023toolformer,shen2023hugginggpt}.
Reflection, memory, planning, reusable skills, and multi-agent
collaboration supported adaptation
from experience and sustained interaction
\citep{shinn2023reflexion,park2023generative,wang2023voyager,wu2024autogen,hong2024metagpt},
establishing core interaction patterns for agents operating over
multiple turns.

Modern coding agents and execution harnesses integrate repository
editing, code execution, testing, and browsing to resolve software
issues and complete practical web and research tasks
\citep{yang2024sweagent,zhang2024autocoderover,antoniades2025swesearch,wang2026openhands_sdk,soni2026openhands_versa}.
Coding agents also support algorithm discovery
\citep{novikov2025alphaevolve}.
Agents also use execution feedback and accumulated experience to
refine their tools, workflows, and memory
\citep{hu2025adas,robeyns2025sica,xia2025livesweagent,zhang2026dgm,ren2026tmem}.
These systems motivate optimizing agents across complete interaction
turns and the tokens generated within each turn.

\subsection{RL in LLM Posttraining}
\label{sec:related-post-training}

Classical policy-gradient methods optimize expected return through
reward- or advantage-weighted updates
\citep{williams1992reinforce,sutton1999policygradient}.
Advantage estimation and conservative, trust-region, or proximal updates
provide foundations for stable policy optimization
\citep{kakade2002conservative,schulman2015trpo,schulman2015gae,schulman2017ppo}.

Building on GRPO \citep{shao2024deepseekmath}, DeepSeek-R1
\citep{guo2025deepseekr1} demonstrated the potential of large-scale RL
for reasoning. Subsequent methods improve importance weighting and
trust-region control
\citep{yu2025dapo,zheng2025gspo,gao2025sapo,qi2026dppo},
as well as critic learning and long-response optimization
\citep{yuan2025vcppo,yue2025vapo,fan2025tppo,gong2026sae}.
Risk-based objectives further address exploration and solution diversity
\citep{ren2026riskpo,jiang2026risksensitive}.
Policy optimization also targets visually grounded causal reasoning
\citep{liang2026perception}.

RL also supports multi-turn agents that learn through tool use and
environment feedback
\citep{zhou2024archer,jin2025searchr1,luo2025agentlightning}.
Long horizons and sparse task rewards make credit assignment across
decisions especially important
\citep{feng2025gigpo,dong2026arpo,zhu2026eibench,zhang2026hygae}.
Recent work further develops stable policy updates for asynchronous
agent training \citep{hou2026sao,hu2026flashreinforce}.
MPO aligns advantage estimation and importance correction with complete
turns, while retaining centered token-level credit within each turn.

\section{Agentic RL Across Turns and Tokens}
\label{sec:preliminary}

\subsection{Agentic RL with Terminal Reward}
\label{sec:agentic-setting}

An LLM agent solves a task by alternating between assistant outputs and
environment observations. Its interaction trajectory is
\begin{equation}
 \tau=(o_0,a_0,o_1,a_1,\ldots,o_T,a_T).
 \label{eq:trajectory}
\end{equation}
Here $o_0$ is the task prompt, $a_t$ is an assistant output such as a tool
call or final response, and $o_{t+1}$ is the environment feedback following
a nonterminal output. The trajectory ends after $a_T$. Each observation
and assistant output is a token sequence:
\begin{equation}
 o_t=(o_{t,1},\ldots,o_{t,M_t}),\qquad
 a_t=(a_{t,1},\ldots,a_{t,L_t}),
 \label{eq:token-sequences}
\end{equation}
Here $M_t$ and $L_t\geq1$ are the respective sequence lengths, and
$t=0,\ldots,T$ indexes assistant turns.
Let $h_t=(o_0,a_0,\ldots,o_t)$ denote the history at the start of turn $t$,
with $h_0=(o_0)$.
The history before generating token $a_{t,i}$ is
\begin{equation}
 h_{t,i}=(h_t,a_{t,<i}),\qquad
 a_{t,<i}=(a_{t,1},\ldots,a_{t,i-1}).
 \label{eq:token-history}
\end{equation}
The assistant policy factorizes autoregressively as
$\pi_\theta(a_t\mid h_t)=\prod_{i=1}^{L_t}\pi_\theta(a_{t,i}\mid h_{t,i})$.
Environment observations follow
$o_{t+1}\sim P_{\mathrm{env}}(\cdot\mid h_t,a_t)$ and become part of
the next assistant turn's context. The task distribution and environment
dynamics are independent of $\theta$; together with the policy, they
induce the trajectory distribution $p_\theta$.

The agent receives a single reward $R(\tau)$ when the interaction
terminates, with no intermediate rewards. Its objective is
$\mathcal J(\theta)=\mathbb E_{\tau\sim p_\theta}[R(\tau)]$.
Only the final assistant token of the entire trajectory receives this
reward; every earlier token has zero reward:
\begin{equation}
 r_{t,i}:=\begin{cases}
 R(\tau),&t=T,\ i=L_T,\\
 0,&\text{otherwise}.
 \end{cases}
 \label{eq:terminal-reward}
\end{equation}

\subsection{Turns as Macro Decisions}
\label{sec:macro-setting}

The environment receives complete assistant outputs as actions, while
the policy generates each output autoregressively. We treat $a_t$ as
the macro action and retain $h_t$ as its state: the history includes
the time index and any remaining interaction or token budget. The
turn-ending decision is included in $a_t$. A fixed maximum token length
$L_{\max}$ and any forced termination rule are shared by both policies.
There is no environment interaction between tokens inside one assistant
output. The macro transition is therefore $h_t\xrightarrow{a_t}h_{t+1}$,
with the same environment law defined above.

We use the original policy notation at both scales:
\begin{equation}
 \pi_\theta(a_t\mid h_t)
 =\prod_{i=1}^{L_t}\pi_\theta(a_{t,i}\mid h_{t,i}).
 \label{eq:macro-policy}
\end{equation}
The complete-turn probability on the left is induced by the token
probabilities on the right, using the same model parameters. This
factorization lets us evaluate changes in either a complete output or
an individual token distribution while retaining one policy.
The objective remains
\begin{equation}
 \mathcal J(\theta)=\mathbb E_{\tau\sim p_\theta}[R(\tau)].
 \label{eq:macro-objective}
\end{equation}
When discussing two fixed policies, we also write $\mathcal J(\pi)$
and $\mathcal J(\pi')$ for this same objective.

\begin{lemma}[Objective consistency]
\label{lem:objective-consistency}
Under a lossless token refinement, identical reward and stopping rules,
and $\gamma=1$, the macro and token descriptions have the same task
return. Moreover,
\begin{equation}
 \nabla_\theta\log\pi_\theta(a_t\mid h_t)
 =\sum_{i=1}^{L_t}\nabla_\theta
              \log\pi_\theta(a_{t,i}\mid h_{t,i}).
 \label{eq:macro-score}
\end{equation}
\end{lemma}

The proof is given in Appendix~\ref{app:macro-pdl-proof}.

Both descriptions assign the same probability and return to each
complete interaction. Their common task objective allows us to compare
surrogates built at different scales. The choice of advantage weights
and the choice of importance ratios are separate parts of this
construction; the following analysis makes the surrogate difference
explicit.
The theoretical analysis assumes a fixed turn horizon of $T+1$.

\section{Policy Improvement Across Turns and Tokens}
\label{sec:theoretical-guarantee}

We analyze how turn-level and token-level surrogates relate to the same
task-return objective. Following the surrogate-and-distribution-shift
analysis of TRPO \citep{schulman2015trpo}, we first establish a
finite-horizon turn-level improvement bound, then identify the
additional discrepancy introduced by token surrogates. These results
motivate combining multiscale credit assignment with turn-level
importance weighting in MPO. Full proofs and the extension to estimated
turn advantages are provided in Appendix~\ref{app:improvement-proofs}.

\subsection{Turn-Level Policy Improvement}
\label{sec:exact-improvement}

Fix an old policy $\pi$ and a candidate $\pi'$ on the same turn process,
with horizon $T+1$, $\gamma=1$, and identical initial-history and
environment distributions. Define the time-indexed occupancy measure
at turn $t$ as $d_t^\pi(h):=\Pr_{\pi}(h_t=h)$ for $t=0,\ldots,T$,
where the probability is taken over the initial task, assistant outputs
sampled from $\pi$, and environment responses. Thus $d_t^\pi$ is the
distribution of histories at the start of turn $t$.
At each active history, define the exact old-policy values by
\begin{align}
 V_t^\pi(h)&=\mathbb E_\pi[R(\tau)\mid h_t=h],\notag\\
 Q_t^\pi(h,a)&=\mathbb E_\pi[R(\tau)\mid h_t=h,a_t=a],\notag\\
 A_t^\pi(h,a)&=Q_t^\pi(h,a)-V_t^\pi(h).
 \label{eq:macro-exact-values}
\end{align}
These expectations mean continuing $\pi$ from the supplied history;
the values must also be defined at histories reached by $\pi'$.
Set $V_{T+1}^\pi=0$ and all post-termination values to zero. Assume
all expectations below are finite; bounded task rewards suffice.
Common policy support is assumed for importance ratios, but is not
needed for the performance-difference or TV inequalities themselves.
The Bellman identity gives $\mathbb E_{a\sim\pi}A_t^\pi(h,a)=0$.
\begin{lemma}[Macro performance difference and importance ratios]
\label{lem:macro-pdl}
In the finite-horizon setting above, the performance difference satisfies
\begin{equation}
 \mathcal J(\pi')-\mathcal J(\pi)
 =\sum_{t=0}^{T}
 \mathbb E_{\substack{h_t\sim d_t^{\pi'}\\
                      a_t\sim\pi'(\cdot\mid h_t)}}
 \!\left[A_t^\pi(h_t,a_t)\right].
 \label{eq:macro-pdl}
\end{equation}
Define the surrogate under the old-policy occupancy measure by
$\mathcal S_M(\pi';\pi)
:=\sum_{t=0}^{T}\mathbb E_{d_t^\pi,\pi'}[A_t^\pi]$.
If $\pi'(\cdot\mid h)$ is absolutely continuous with respect to
$\pi(\cdot\mid h)$ at old-policy histories, then
$\mathcal S_M(\pi';\pi)
=\mathbb E_\pi\sum_{t=0}^{T}(\rho_t-1)A_t^\pi$,
where $\rho_t=\pi'(a_t\mid h_t)/\pi(a_t\mid h_t)$ is the
complete-turn importance ratio.
\end{lemma}

$\mathcal S_M$ is an improvement, so $\mathcal S_M(\pi;\pi)=0$;
the corresponding absolute surrogate is $\mathcal J(\pi)+\mathcal S_M$.

The proof is given in Appendix~\ref{app:macro-pdl-proof}.

\begin{theorem}[Monotonic improvement at the turn scale]
\label{thm:macro-improvement}
In the finite-horizon setting above, suppose
$|A_t^\pi(h,a)|\le A_{\max}$ at all admissible turn-start histories
and complete outputs. Set
\[
 \alpha_M=\sup_{t,h}D_{\mathrm{TV}}
       (\pi'(\cdot\mid h),\pi(\cdot\mid h)),
\]
where $D_{\mathrm{TV}}(p,q)=\frac12\|p-q\|_1$.
Here the supremum ranges over turn-start histories, and the distributions
are over complete outputs $a_t$.
Then
\begin{equation}
 \mathcal J(\pi')-\mathcal J(\pi)\geq
 \mathcal S_M-2A_{\max}T(T+1)\alpha_M^2.
 \label{eq:exact-macro-bound}
\end{equation}
Thus a nonnegative right-hand side certifies nondecrease of the task
return; a positive right-hand side certifies strict improvement.
\end{theorem}

Appendices~\ref{app:macro-pdl-proof}--\ref{app:macro-bound-proof}
provide the full performance-difference derivation and proof of
Theorem~\ref{thm:macro-improvement}.

The bound certifies improvement when the exact surrogate gain covers
the distribution-shift penalty. Its surrogate combines a turn-level
advantage with a complete-output importance ratio, connecting credit
at interaction boundaries to changes in the probability of each turn.

\subsection{Additional Cost of Token Surrogates}
\label{sec:scale-conversion}

The autoregressive parameterization supports the macro surrogate through
$\rho_t=\prod_i\rho_{t,i}$, where
$\rho_{t,i}=\pi'(a_{t,i}\mid h_{t,i})/\pi(a_{t,i}\mid h_{t,i})$.
The product ratio corrects the entire turn distribution. A separate
token ratio corrects only the next-token distribution at an old-policy
prefix, so replacing the turn surrogate with token terms changes the
prefix distributions used to evaluate credit.

\begin{theorem}[The additional cost of token surrogates]
\label{thm:scale-penalties}
Let $A_{t,i}^{\pi,\mathrm{tok}}$ be the exact token advantage, and let
$A_{\max}^{\mathrm{tok}}<\infty$ be a uniform bound on its absolute value
over all admissible token prefixes and next tokens, so that
$|A_{t,i}^{\pi,\mathrm{tok}}(h_{t,i},a_{t,i})|\le A_{\max}^{\mathrm{tok}}$.
Set
$\alpha_T=\sup_h D_{\mathrm{TV}}(\pi'(\cdot\mid h),\pi(\cdot\mid h))$,
where $h$ ranges over token prefixes and the distributions are over next
tokens. Define
\[
 \mathcal S_T=\mathbb E_\pi\sum_t\sum_i
              (\rho_{t,i}-1)A_{t,i}^{\pi,\mathrm{tok}}.
\]
For $L_t\le L_{\max}$,
\begin{equation}
 |\mathcal S_M-\mathcal S_T|
 \le2(T+1) A_{\max}^{\mathrm{tok}}L_{\max}(L_{\max}-1)\alpha_T^2.
 \label{eq:within-turn-shift}
\end{equation}
Consequently,
\begin{equation}
 \begin{aligned}
 \mathcal J(\pi')-\mathcal J(\pi)\ge{}&\mathcal S_T
       -2A_{\max}T(T+1)\alpha_M^2\\
       &-2(T+1) A_{\max}^{\mathrm{tok}}L_{\max}(L_{\max}-1)\alpha_T^2.
 \end{aligned}
 \label{eq:token-bridge-bound}
\end{equation}
\end{theorem}

At a fixed macro state, the exact macro advantage telescopes into
exact token advantages. The macro surrogate averages them over
new-policy prefixes, whereas $\mathcal S_T$ uses old-policy prefixes.
This within-turn distribution shift produces the additional term;
Appendix~\ref{app:within-turn-proof} gives the proof.

The policy-change budgets must also be compared at the same scale.
By the KL chain rule, a uniform token KL cap $\delta$, including the
stopping decision, implies
\begin{equation}
 D_{\mathrm{KL}}^{\max}(\pi\Vert\pi')\le L_{\max}\delta,
 \qquad \alpha_M^2\le\frac{L_{\max}\delta}{2}.
 \label{eq:token-to-macro-kl}
\end{equation}
For $R\in[0,1]$, setting $\delta=\Delta/L_{\max}$ certifies a macro
KL budget $\Delta$; the token bound then carries an additional penalty
$(T+1)(L_{\max}-1)\Delta$.
\begin{tcolorbox}[
 colback=blue!4!white, colframe=blue!45!black,
 boxrule=0.6pt, arc=2mm,
 left=6pt, right=6pt, top=5pt, bottom=5pt,
 before skip=7pt, after skip=7pt]
\textbf{Takeaway.}
Under the same KL constraints, the token-surrogate improvement
bound incurs an additional within-turn distribution-shift penalty
relative to the turn-surrogate bound.
\end{tcolorbox}

% Advantage-estimation semantics follow search-bcp-rl/verl/trainer/ppo/core_algos.py.
% t indexes assistant turns; i indexes generated tokens within a turn.
\section{Multiscale Policy Optimization}
\label{sec:mpo}
\label{sec:mae}
\label{sec:estimation-scales}

As illustrated in Figure~\ref{fig:mpo-overview}, MPO combines two components
motivated by the analysis above. Its
advantage estimator, MAE, propagates credit across turns and allocates
it among tokens within each turn
(Sections~\ref{sec:turn-advantage}--\ref{sec:intra-advantage}). Its policy
update uses complete-turn importance weights while preserving each
token's advantage in the gradient (Section~\ref{sec:mae-training}).

The two MAE stages reuse the critic $V_{t,i}=V(h_{t,i})$, whose old-policy
rollout predictions remain fixed during advantage estimation.
Environment tokens enter the history but are not policy actions.
At turn boundaries, set
\begin{equation}
 V_{t,L_t+1}:=V_{t+1,1}\quad(0\leq t\leq T),\qquad
 V_{T+1,1}:=0,
 \label{eq:value-boundary}
\end{equation}
with zero values after true termination. Under the terminal reward
placement in Eq.~\eqref{eq:terminal-reward} and $\gamma=1$, define
\begin{equation}
 \delta_{t,i}=r_{t,i}+V_{t,i+1}-V_{t,i}.
 \label{eq:td}
\end{equation}

\subsection{Turn-Level Credit Propagation}
\label{sec:turn-advantage}
\label{sec:macro-gae-theory}

Summing TD residuals within a turn cancels all internal value predictions:
\begin{equation}
 \delta_t^{\mathrm{turn}}
 =\sum_{i=1}^{L_t}\delta_{t,i}
 =\mathbf 1\{t=T\}R(\tau)+V_{t+1,1}-V_{t,1}.
 \label{eq:turn-td}
\end{equation}
Applying GAE \citep{schulman2015gae} on this macro time axis gives
\begin{equation}
 \begin{aligned}
 A_t^{\mathrm{turn}}
 &=\delta_t^{\mathrm{turn}}+\lambda A_{t+1}^{\mathrm{turn}}\\
 &=\sum_{\ell=0}^{T-t}\lambda^\ell\delta_{t+\ell}^{\mathrm{turn}},
 \end{aligned}
 \label{eq:turn-gae}
\end{equation}
where $\lambda\in[0,1]$ and $A_{T+1}^{\mathrm{turn}}=0$.
This shared credit decays once per interaction rather than once per
token. For fixed rewards and boundary values, it is independent of
turn lengths and internal critic predictions.
Appendix~\ref{app:turn-credit-error} characterizes the corresponding
estimation error using only boundary critic errors; the token-GAE
comparison is in Appendix~\ref{sec:gae}.

\subsection{Within-Turn Credit Allocation}
\label{sec:intra-advantage}

To refine the relative token weights, form a local trace restricted
to the current turn:
\begin{equation}
 A_{t,i}^{\mathrm{intra}}
 =\sum_{j=i}^{L_t}\lambda_{\mathrm{intra}}^{j-i}\delta_{t,j}.
 \label{eq:intra-gae}
\end{equation}
Only the advantage trace is cut at the boundary: the final TD residual
still bootstraps from $V_{t+1,1}$, and equals $R(\tau)-V_{T,L_T}$
at trajectory termination.

We combine the turn and local signals according to two design
requirements. First, the turn estimate should determine the mean token
credit. It propagates task-level credit across interactions using boundary
values, so we use it as the shared credit for the complete output.
Broadcasting this estimate gives the turn-level score contribution by
Eq.~\eqref{eq:macro-score}; local refinement then adjusts individual token
weights around that common level. Directly adding the local trace would
also shift the mean by $\beta\overline A_t^{\mathrm{intra}}$, where
$\overline A_t^{\mathrm{intra}}=L_t^{-1}\sum_iA_{t,i}^{\mathrm{intra}}$.
Fixing the mean to $A_t^{\mathrm{turn}}$ keeps the local branch from
changing this shared credit.

Second, the allocation should retain the relative differences in the
local estimates, which provide information for distinguishing tokens
within a turn. Broadcasting alone assigns identical weights to all
tokens. We therefore seek weights that satisfy the prescribed turn mean
while remaining as close as possible to the scaled local signals.
Writing $\mathbf A_t^{\mathrm{intra}}$ for their vector, we express these
requirements as a least-squares projection. The squared distance treats
all token positions equally and penalizes changes to the local weights:
\begin{equation}
 \begin{aligned}
 \mathbf A_t^{\mathrm{MAE}}
 &=\underset{\mathbf w\in\mathbb R^{L_t}}{\operatorname{argmin}}
       \frac12\|\mathbf w-\beta\mathbf A_t^{\mathrm{intra}}\|_2^2,\\
 &\hspace{1em}\text{subject to}\quad
       \frac{1}{L_t}\sum_{i=1}^{L_t}w_i=A_t^{\mathrm{turn}},
 \end{aligned}
 \label{eq:mae-projection}
\end{equation}
where $\beta\geq0$ scales the local refinement. Solving this strictly
convex problem yields the unique allocation
\begin{equation}
 \boxed{A_{t,i}^{\mathrm{MAE}}=A_t^{\mathrm{turn}}
 +\beta\bigl(A_{t,i}^{\mathrm{intra}}-\overline A_t^{\mathrm{intra}}\bigr)}.
 \label{eq:mae}
\end{equation}
Before batch normalization, it satisfies
\begin{equation}
 \frac{1}{L_t}\sum_{i=1}^{L_t}A_{t,i}^{\mathrm{MAE}}=A_t^{\mathrm{turn}}.
 \label{eq:mae-mean}
\end{equation}
Centering fixes the average credit to the turn estimate while preserving
pairwise differences in the scaled local signals. Setting $\beta=0$
recovers turn broadcasting.

\subsection{Turn-Level Importance Correction}
\label{sec:mae-training}

Having assigned token credit, we combine it with importance correction
at the turn scale. We normalize $A_{t,i}^{\mathrm{MAE}}$ once over valid
assistant tokens in the rollout batch and hold the resulting weights
$\widehat A_{t,i}^{\mathrm{MAE}}$ fixed during the update. Write
$\pi_\theta$ for the candidate policy and $\pi$ for the fixed rollout
policy. The token and complete-turn ratios are
\begin{equation}
 \rho_{t,i}(\theta)
 =\frac{\pi_\theta(a_{t,i}\mid h_{t,i})}{\pi(a_{t,i}\mid h_{t,i})},
 \qquad
 \rho_t(\theta)=\prod_{i=1}^{L_t}\rho_{t,i}(\theta).
 \label{eq:mae-update-ratios}
\end{equation}
Multiplying all MAE advantages by the same differentiable $\rho_t$
would cancel the centered local correction in the unclipped sum.
To retain token-specific credit under complete-turn importance
weighting, we stop the gradient through the other token ratios:
\begin{equation}
 \widetilde\rho_{t,i}(\theta)
 =\rho_{t,i}(\theta)\,\operatorname{sg}\!\left[
       \prod_{\substack{j=1\\j\ne i}}^{L_t}\rho_{t,j}(\theta)
   \right],
 \label{eq:mae-turn-corrected-ratio}
\end{equation}
where $\operatorname{sg}$ preserves its argument in the forward pass
and has zero gradient. Thus every token has the complete-turn weight
$\widetilde\rho_{t,i}(\theta)=\rho_t(\theta)$ in the forward pass, while
\begin{equation}
 \nabla_\theta\widetilde\rho_{t,i}(\theta)
 =\rho_t(\theta)\nabla_\theta
       \log\pi_\theta(a_{t,i}\mid h_{t,i}).
 \label{eq:mae-turn-corrected-gradient}
\end{equation}
This construction weights each token's score gradient by the complete
turn ratio and its own MAE advantage. The MPO update is defined by the
clipped surrogate
\begin{equation}
 \begin{aligned}
 \mathcal L_{\mathrm{MPO}}(\theta)
 &=\widehat{\mathbb E}_{t,i}\Bigl[\min\bigl(
       \widetilde\rho_{t,i}(\theta)\widehat A_{t,i}^{\mathrm{MAE}},\\
       &\qquad\operatorname{clip}\bigl(\widetilde\rho_{t,i}(\theta),
          1-\epsilon,1+\epsilon\bigr)
          \widehat A_{t,i}^{\mathrm{MAE}}\bigr)\Bigr],
 \end{aligned}
 \label{eq:mae-turn-is-objective}
\end{equation}
where $\widehat{\mathbb E}_{t,i}$ is the empirical mean over all valid
assistant tokens in the rollout batch and $\epsilon>0$ is the clipping
threshold. Gradient ascent uses the stop-gradient rule in
Eq.~\eqref{eq:mae-turn-corrected-ratio}; the full turn ratio therefore
controls the forward weighting and clipping, while each token's MAE
advantage controls its own score gradient.

Following VAPO's decoupled GAE~\citep{yue2025vapo}, we set
$\lambda_{\mathrm{critic}}=1$, so the critic directly regresses the
terminal return $R(\tau)$.

\begin{tcolorbox}[
 colback=blue!4!white, colframe=blue!45!black,
 boxrule=0.6pt, arc=2mm,
 left=6pt, right=6pt, top=5pt, bottom=5pt,
 before skip=7pt, after skip=7pt]
\textbf{Takeaway.}
MPO propagates credit across turns and refines it within each turn
while preserving the turn mean. Turn-level importance weighting
then retains token-specific credit in the policy gradient.
\end{tcolorbox}

\section{Experiments}
\label{sec:experiments}
\subsection{Experimental Settings}
\label{sec:experimental-settings}
We evaluate Qwen3-4B-Thinking and Qwen3-30B-A3B-Instruct
\citep{qwen2025technical} in three agentic settings: Search, Claw-Eval,
and Automation Bench. The main
comparisons use separate training in each setting, with 300 steps for
Search, 300 for Claw-Eval, and 200 for Automation Bench. We compare
MPO with the initial checkpoint (Base), GRPO~\citep{shao2024deepseekmath},
GSPO~\citep{zheng2025gspo}, PPO~\citep{schulman2017ppo},
VAPO~\citep{yue2025vapo}, and SAO~\citep{hou2026sao}.
To examine the update objective, MPO w/ token IS retains MAE but
replaces complete-turn importance ratios with token-level ratios.

\paragraph{Evaluation protocol.}
For the trained-model results in all tables, we select the
best-performing checkpoint from each training run. We evaluate the
selected checkpoint offline five times and report the mean of the
five evaluation scores. The benchmark-specific rollout budgets and
metrics described below apply to each evaluation repetition.
For BrowseComp-Plus, we follow the held-out evaluation setup of
\citet{lu2025supo}. For GAIA and HLE, we use WebThinker's text-only
evaluation subsets, drawn from their validation and test sets,
respectively \citep{li2025webthinker}. Base results use three rollouts
per task. For trained Search models, pass@1 and pass@3 are estimated
from a pool of 32 independent rollouts per task, using the same samples
across evaluation budgets. We follow the original evaluation protocols
for Claw-Eval \citep{ye2026claweval} and Automation Bench
\citep{shepard2026automationbench}. Claw-Eval uses three independent
trials, reporting pass@3 for success in at least one trial and
pass$^{3}$ for success in all three. Automation Bench reports pass@1
using its end-state assertion checks. All score differences below are
percentage points.

\subsection{Main Results}
\label{sec:main-results}
\begin{table}[t]
 \centering
 \caption{Main evaluation on search and agent tasks. Search benchmarks
 report pass@1 and pass@3; Claw-Eval reports pass@3 and pass$^{3}$.
 Automation Bench reports pass@1.
 Search and Claw-Eval scores are percentages; higher is better for all metrics.
 Base denotes the initial checkpoint before task-specific RL. Trained-model
 Search pass@1 and pass@3 are estimated using 32 rollouts per task.}
 \label{tab:main-results}
 \footnotesize
 \setlength{\tabcolsep}{2.5pt}
 \renewcommand{\arraystretch}{1.12}
 \begin{tabularx}{\textwidth}{@{}l*{8}{>{\hsize=.95\hsize\centering\arraybackslash}X}>{\hsize=1.4\hsize\centering\arraybackslash}X@{}}
  \toprule
  & \multicolumn{2}{c}{BrowseComp-Plus}
  & \multicolumn{2}{c}{GAIA}
  & \multicolumn{2}{c}{HLE}
  & \multicolumn{2}{c}{Claw-Eval}
  & \shortstack{Automation\\Bench} \\
  \cmidrule(lr){2-3}\cmidrule(lr){4-5}\cmidrule(lr){6-7}
  \cmidrule(lr){8-9}\cmidrule(lr){10-10}
  Method & pass@1 & pass@3 & pass@1 & pass@3 & pass@1 & pass@3
  & pass@3 & pass$^{3}$ & pass@1 \\
  \midrule
  \multicolumn{10}{l}{\textbf{Qwen3-4B-Thinking}} \\
  Base & 8.84  & 18.07 & 24.92 & 33.98 & 5.53 & 9.60 & 25.22 & 8.09 & 0.00 \\
  \rowcolor{policyrow}
  GRPO & 39.34 & 52.12 & 33.92 & 44.63 & 8.51 & 14.43 & 64.45 & 48.79 & 4.33 \\
  \rowcolor{policyrow}
  GSPO & 38.56 & 52.11 & 34.34 & 45.26 & 9.26 & 13.71 & 63.83 & 49.35 & 5.22 \\
  \rowcolor{criticrow}
  PPO  & 40.85 & 54.30 & 31.71 & 45.03 & 7.22 & 12.82 & 63.25 & 50.78 & 11.67 \\
  \rowcolor{criticrow}
  VAPO & 43.23 & 57.33 & 33.91 & 46.15 & 7.54 & 13.27 & 64.35 & 49.22 & 11.50 \\
  \rowcolor{criticrow}
  SAO  & 43.75 & 56.88 & 35.62 & 47.56 & 8.19 & 14.23 & 64.52 & 49.11 & 12.06\\
  \rowcolor{maerow}
  MPO w/ token IS & 45.72 & 57.47 & 36.35 & 50.27 & 8.89 & 15.17 & 65.13 & 50.22 & 12.45 \\
  \rowcolor{maerow}
  \textbf{MPO (Ours)} & 48.08 & 61.83 & 39.47 & 52.67 & 9.92 & 16.78 & 66.91 & 51.73 & 13.33 \\
  \midrule
  \multicolumn{10}{l}{\textbf{Qwen3-30B-A3B-Instruct}} \\
  Base & 21.69 & 27.71 & 25.24 & 37.86 & 7.47 & 12.60 & 32.19 & 12.64 & 0.17 \\
  \rowcolor{policyrow}
  GRPO & 41.83 & 48.18 & 31.03 & 40.51 & 8.33 & 12.68 & 68.23 & 52.55 & 5.00 \\
  \rowcolor{policyrow}
  GSPO & 40.22 & 47.31 & 35.25 & 46.25 & 8.61 & 13.34 & 68.62 & 51.34 & 6.23\\
  \rowcolor{criticrow}
  PPO  & 45.20 & 51.63 & 35.70 & 45.28 & 8.31 & 13.33 & 70.14 & 55.33 & 14.17 \\
  \rowcolor{criticrow}
  VAPO & 46.73 & 57.33 & 38.33 & 49.22 & 9.25 & 14.64 & 71.31 & 56.72 & 13.77\\
  \rowcolor{criticrow}
  SAO  & 45.36 & 57.88 & 39.62 & 50.34 & 9.91 & 14.37 & 72.72 & 57.23 & 14.20 \\
  \rowcolor{maerow}
  MPO w/ token IS & 48.27 & 58.12 & 41.23 & 52.12 & 9.45 & 14.67 & 72.19 & 57.61 & 13.17 \\
  \rowcolor{maerow}
  \textbf{MPO (Ours)} & 50.32 & 59.33 & 43.42 & 56.24 & 10.03 & 15.16 & 73.22 & 58.35 & 15.00 \\
  \bottomrule
 \end{tabularx}
\end{table}

% Base search results: search_eval/runs/local_initial_pass3_20261008/results.csv
% Metric precedents:
% WebDancer, https://arxiv.org/abs/2505.22648 (GAIA pass@3).
% Chain-of-Agents, https://arxiv.org/abs/2508.13167 (GAIA/HLE pass@3).
% Claw-Eval, https://arxiv.org/abs/2604.06132 (pass@3 and pass^3).
% TODO: Specify held-out splits, text-only subsets where applicable, and
% matched tool, decoding, and trajectory budgets before reporting results.

\paragraph{Performance across tasks and model scales.}
Table~\ref{tab:main-results} shows that MPO achieves the highest reported
score on all nine metrics at both model sizes. Relative to the strongest
non-MPO baseline in each column, MPO improves BrowseComp-Plus, GAIA, and HLE
pass@1 by 4.33, 3.85, and 0.66 points on 4B, and by 3.59, 3.80, and
0.12 points on 30A3B. The gains extend to pass@3, including improvements
of 5.11 and 5.90 points on GAIA at the two scales. MPO also improves
both Claw-Eval metrics and Automation Bench performance. For example,
Claw-Eval pass$^{3}$ increases by 0.95 and 1.12 points, showing that
the improvements extend to success across repeated attempts.

\paragraph{Aligning advantage estimation and importance correction.}
Replacing turn-level ratios with token-level ratios reduces every
reported score. The Search pass@1 drops on BrowseComp-Plus, GAIA,
and HLE are 2.36, 3.12, and 1.03 points for 4B, and 2.05, 2.19,
and 0.58 points for 30A3B. With token IS, MAE still outperforms PPO
on all six Search metrics at both scales, indicating that multiscale
credit assignment contributes useful gains on its own. Complete-turn
importance correction provides a further improvement. Together, these
comparisons support aligning shared credit and probability correction
with the turn-level interaction objective, while retaining local
token distinctions in the policy gradient.
Figure~\ref{fig:token-is-training} provides a visualization
of validation performance and training reward through step 300.

\par\medskip
\noindent\begin{minipage}{\columnwidth}
 \centering
 \captionsetup{hypcap=false}
 \includegraphics[width=0.49\linewidth]{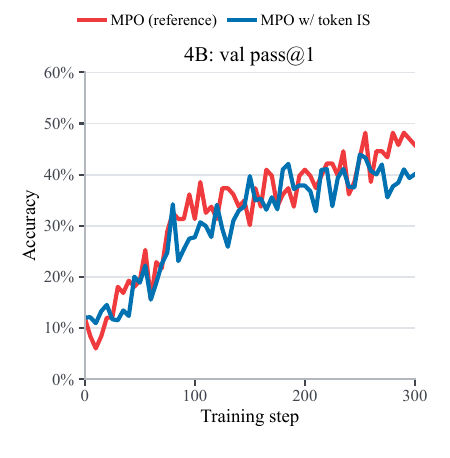}\hfill
 \includegraphics[width=0.49\linewidth]{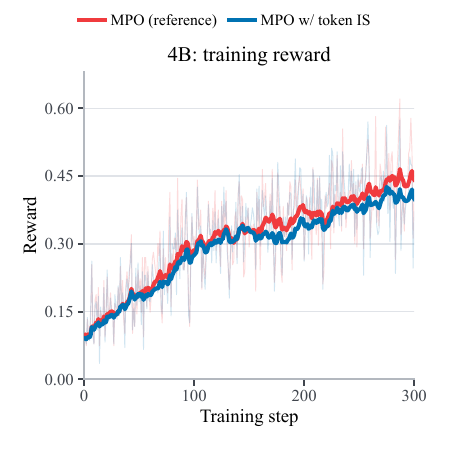}
 \captionof{figure}{4B validation pass@1 (left) and training reward
 (right) through step 300.}
 \label{fig:token-is-training}
\end{minipage}
\par\medskip

\begin{figure}[t]
 \centering
 \includegraphics[width=0.98\textwidth]{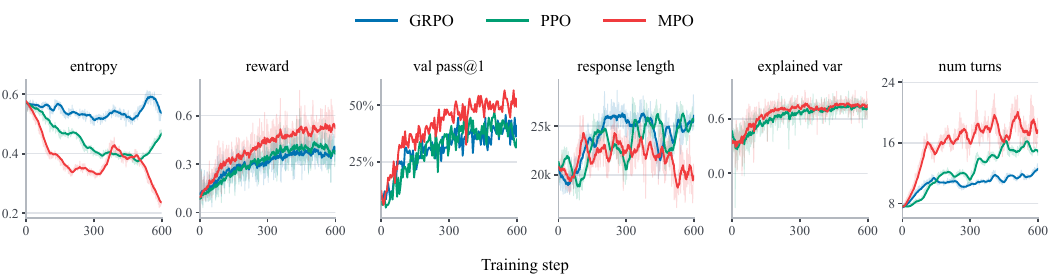}
 \caption{Training dynamics of GRPO, PPO, and MPO on 4B through step 600,
 showing entropy, reward, validation pass@1, response length,
 explained variance, and number of turns.}
 \label{fig:training-dynamics}
\end{figure}

\subsection{Training Dynamics}
\label{sec:training-dynamics}

Figures~\ref{fig:training-dynamics} and~\ref{fig:training-dynamics-30a3b}
compare GRPO, PPO, and the full MPO method, combining MAE with
turn-level importance correction. MPO is shown in red in these
figures and Figure~\ref{fig:test-time-scaling}.
We track entropy, reward, validation pass@1, response length, critic
explained variance, and interaction turns through step 600 for 4B
and step 300 for 30A3B. Each method has one run per model size.

\begin{figure}[b]
 \centering
 \includegraphics[width=0.98\textwidth]{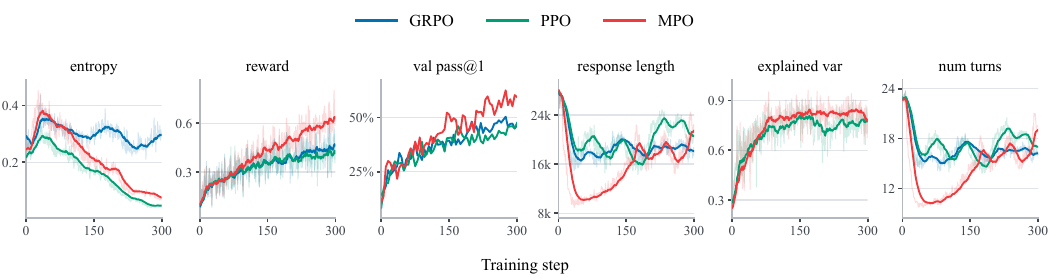}
 \caption{Training dynamics of GRPO, PPO, and MPO on 30A3B through step 300,
 using the same metrics as Figure~\ref{fig:training-dynamics}.}
 \label{fig:training-dynamics-30a3b}
\end{figure}

\paragraph{Critic fit under the same regression target.}
PPO and MPO both train the critic against terminal returns with
$\lambda_{\mathrm{critic}}=1$, yet MPO exhibits higher explained variance
on average during late training at both model scales. A plausible
explanation is that turn-aligned policy updates change the visited histories and
subsequent outcomes, making their return structure easier for the
critic to capture. Better value estimates may in turn improve the
actor's credit signal. This interpretation concerns the interaction
between policy learning and critic fitting: explained variance is
measured on each policy's own rollouts and also depends on return
variance, so it does not isolate a causal improvement in critic accuracy.

\subsection{Test-Time Scaling}
\label{sec:test-time-scaling}

\begin{figure}[t]
 \centering
 \includegraphics[width=0.98\textwidth]{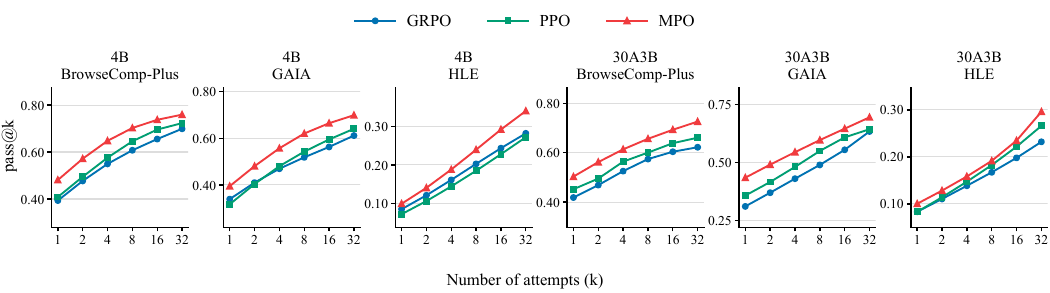}
 \caption{Test-time scaling of GRPO, PPO, and MPO on BrowseComp-Plus, GAIA,
 and HLE for 4B and 30A3B, showing pass@k as the number of attempts increases.}
 \label{fig:test-time-scaling}
\end{figure}

Figure~\ref{fig:test-time-scaling} compares GRPO, PPO, and MPO on
BrowseComp-Plus, GAIA, and HLE at both 4B and 30A3B scales. We report
pass@k for $k\in\{1,2,4,8,16,32\}$ using 32 independent rollouts per
task, reusing the same pool for every $k$. All MPO curves use the
full method with turn-level importance correction. Each method uses
a fixed checkpoint across sampling budgets. Decoding settings,
available tools, and per-rollout budgets are matched across methods.

\paragraph{Benefits beyond a single attempt.}
The gains from MPO extend to larger sampling budgets.
On 4B, MPO leads GRPO and PPO at every evaluated $k$ on all three
benchmarks. Its pass@32 reaches 75.90 on BrowseComp-Plus, 69.90 on
GAIA, and 34.00 on HLE, exceeding the stronger of GRPO and PPO
by 3.61, 5.82, and 5.80 points, respectively. Thus the advantage
persists as additional samples expand task coverage. The
30A3B runs also benefit from additional attempts, although the
relative gains depend on the benchmark and sampling budget.

These results also put the entropy curves in context. PPO's entropy
drops sharply during 30A3B training while its validation performance
continues to improve. Likewise, the lower entropy of 4B MPO coexists
with strong performance across sampling budgets. Token entropy alone
therefore does not establish a collapse in useful exploration:
pass@k directly tests whether additional sampled trajectories solve
tasks that a single attempt misses.

\subsection{Ablation Studies}
\label{sec:ablations}

We ablate cross-turn propagation, within-turn trace length, and local
refinement strength on 4B Search by varying $\lambda$,
$\lambda_{\mathrm{intra}}$, and $\beta$ one at a time. The critic
target, initialization, and remaining training setup are held fixed.
Tables~\ref{tab:ablation-credit} and~\ref{tab:ablation-beta} report
pass@1 and pass@3 on all three Search benchmarks. The reference
setting is $(\lambda,\lambda_{\mathrm{intra}},\beta)=(0.80,0.95,1.00)$
and reproduces MPO's 4B results in Table~\ref{tab:main-results}.
Each setting has one run.
Validation learning curves for these settings appear in
Appendix~\ref{app:ablation-validation-curves}.

% Single-column ablation tables near the ablation discussion.
\begin{table}[!t]
 \captionsetup{skip=4pt}
 \renewcommand{\arraystretch}{1.05}
 \centering
 \caption{Search ablation over turn-scale propagation and within-turn trace.
 Scores are percentages. The bold setting reproduces 4B MPO from
 Table~\ref{tab:main-results}. We vary
 $(\lambda,\lambda_{\mathrm{intra}})$ with $\beta=1.00$ fixed.}
 \label{tab:ablation-credit}
 \small
 \setlength{\tabcolsep}{1.2pt}
 \begin{tabular*}{\columnwidth}{@{\extracolsep{\fill}}l*{6}{c}@{}}
  \toprule
  Setting & \multicolumn{2}{c}{BrowseComp-Plus}
  & \multicolumn{2}{c}{GAIA} & \multicolumn{2}{c}{HLE} \\
  \cmidrule(lr){2-3}\cmidrule(lr){4-5}\cmidrule(lr){6-7}
  $(\lambda,\lambda_{\mathrm{intra}})$ & pass@1 & pass@3
  & pass@1 & pass@3 & pass@1 & pass@3 \\
  \midrule
  \textbf{(0.80, 0.95)} & 48.08 & 61.83 & 39.47 & 52.67 & 9.92 & 16.78 \\
  (0.60, 0.95) & 47.51 & 59.22 & 38.92 & 50.79 & 9.11 & 15.17 \\
  (0.85, 0.95) & 46.36 & 58.91 & 37.26 & 51.04 & 8.87 & 14.89 \\
  (0.80, 0.90) & 46.25 & 57.81 & 37.92 & 50.35 & 8.15 & 15.22 \\
  (0.80, 1.00) & 47.67 & 59.18 & 38.23 & 50.95 & 8.43 & 16.16 \\
  \bottomrule
 \end{tabular*}
\end{table}

\begin{table}[!t]
 \captionsetup{skip=4pt}
 \renewcommand{\arraystretch}{1.05}
 \centering
 \caption{Search ablation over local refinement strength. Scores are
 percentages. The bold setting reproduces 4B MPO from
 Table~\ref{tab:main-results}. We vary $\beta$ with
 $\lambda=0.80$ and $\lambda_{\mathrm{intra}}=0.95$ fixed.}
 \label{tab:ablation-beta}
 \small
 \setlength{\tabcolsep}{1.2pt}
 \begin{tabular*}{\columnwidth}{@{\extracolsep{\fill}}l*{6}{c}@{}}
  \toprule
  Setting & \multicolumn{2}{c}{BrowseComp-Plus}
  & \multicolumn{2}{c}{GAIA} & \multicolumn{2}{c}{HLE} \\
  \cmidrule(lr){2-3}\cmidrule(lr){4-5}\cmidrule(lr){6-7}
  $\beta$ & pass@1 & pass@3 & pass@1 & pass@3 & pass@1 & pass@3 \\
  \midrule
  0.0 & 41.80 & 56.77 & 34.20 & 48.90 & 7.88 & 12.96 \\
  0.5 & 44.13 & 58.10 & 37.24 & 51.85 & 8.26 & 13.22 \\
  \textbf{1.0} & 48.08 & 61.83 & 39.47 & 52.67 & 9.92 & 16.78 \\
  1.5 & 46.23 & 60.35 & 38.26 & 51.26 & 8.93 & 16.23 \\
  2.0 & 45.14 & 59.75 & 37.05 & 50.01 & 8.11 & 15.43 \\
  \bottomrule
 \end{tabular*}
\end{table}

\paragraph{Turn-scale propagation, $\lambda$.}
The intermediate setting $\lambda=0.80$ performs best among the
tested values. Increasing it to 0.85 reduces BrowseComp-Plus, GAIA,
and HLE pass@1 by 1.72, 2.21, and 1.05 points, respectively, and
also reduces all three pass@3 scores. Moving closer to one thus
hurts performance over this part of the tested range, consistent
with a longer trace carrying more variability from distant
interactions. Reducing $\lambda$ to 0.60 also lowers performance,
suggesting that shortening the trace too much discards useful
cross-turn credit. The results favor balancing these effects rather
than maximizing the propagation range.

\paragraph{Within-turn trace, $\lambda_{\mathrm{intra}}$.}
Reducing $\lambda_{\mathrm{intra}}$ from 0.95 to 0.90 lowers pass@1
by 1.83, 1.55, and 1.77 points on BrowseComp-Plus, GAIA, and HLE,
and lowers pass@3 on every benchmark. This shorter trace limits how
far local TD information reaches within a generated turn. Setting
$\lambda_{\mathrm{intra}}=1$ also reduces all six scores relative to
0.95, including a 0.41-point drop in BrowseComp-Plus pass@1.
Thus, 0.95 performs best on all six metrics in this sweep.

\paragraph{Different decay rates for different time scales.}
The two coefficients count different events: $\lambda$ decays once
per interaction, whereas $\lambda_{\mathrm{intra}}$ decays once per
generated token. A trajectory contains far fewer interactions than
tokens, and each interaction itself requires many generated tokens.
Using the same coefficient would therefore impose very different
credit ranges. For example, the nominal geometric trace length
$1/(1-\lambda)$ is five turns at $\lambda=0.80$, while
$\lambda_{\mathrm{intra}}=0.95$ gives twenty tokens; reducing the
latter to 0.90 halves that local range. A smaller turn coefficient
can therefore retain credit across several interactions, while a
larger local coefficient avoids restricting allocation to only a
few tokens. This difference motivates tuning the two scales
separately and explains our choice of
$\lambda<\lambda_{\mathrm{intra}}$: multiscale policy optimization
uses distinct propagation ranges for environment interaction and
autoregressive generation.

\paragraph{Local refinement strength, $\beta$.}
Setting $\beta=0$ removes within-turn credit allocation from MAE,
leaving only turn-level GAE: every token in a turn receives the same
turn advantage. Even without within-turn allocation, this turn-level
variant outperforms the standard PPO baseline with token-level GAE
in Table~\ref{tab:main-results} on all six Search metrics. Its pass@1
gains over PPO are 0.95, 2.49, and 0.66 percentage points on
BrowseComp-Plus, GAIA, and HLE, respectively. However, removing local
allocation still substantially reduces performance relative to full
MPO (Table~\ref{tab:ablation-beta}). Compared with $\beta=1.00$, pass@1
drops by 6.28, 5.27, and 2.04 percentage points on BrowseComp-Plus,
GAIA, and HLE, respectively, and pass@3 also decreases on all three
benchmarks. These results show that allocating credit within a turn
provides benefits beyond propagating credit across turns alone.
Introducing local refinement with $\beta=0.50$ partially recovers
the loss, while $\beta=1.00$ performs best on all six metrics.
Increasing $\beta$ further to 1.50 and then 2.00 progressively reduces
both pass@1 and pass@3 across all three benchmarks. Since the local
correction is centered, these changes preserve the turn's mean
advantage while amplifying differences among its tokens. The results
therefore support within-turn credit allocation with a moderate
refinement strength; removing it or giving it excessive weight both
hurt performance.

\section{Conclusion}
\label{sec:conclusion}

We introduced MPO to organize LLM-agent policy optimization across
turns and tokens. Its MAE component combines cross-turn credit
propagation with mean-preserving local allocation, and its update
aligns complete-turn importance correction with token-specific
gradients. Our analysis motivates these scale choices, while
experiments show gains across Search and agent tasks and continued
benefits from increased sampling. The ablations support treating
interaction and generation as distinct credit horizons. Together,
these findings motivate using both scales when optimizing
long-horizon LLM agents.

% Keep experimental floats before the bibliography without forcing a new page.
\FloatBarrier
\bibliography{ref}

\clearpage
\appendix

\section{Proofs of the Policy-Improvement Guarantees}
\label{app:improvement-proofs}

\subsection{Objective Consistency and the Macro PDL}
\label{app:macro-pdl-proof}

\begin{proof}[Proof of Lemma~\ref{lem:objective-consistency}]
Fix a complete interaction $\tau$. Its probability is the initial-task
probability, multiplied by the probabilities of the generated assistant
tokens and the intervening environment responses. Grouping the assistant
factors within each turn gives
\[
 p_\theta(\tau)
 =p(o_0)\prod_{t=0}^{T}\prod_{i=1}^{L_t}
       \pi_\theta(a_{t,i}\mid h_{t,i})
       \prod_{t=0}^{T-1}P_{\mathrm{env}}(o_{t+1}\mid h_t,a_t)
 =p(o_0)\prod_{t=0}^{T}\pi_\theta(a_t\mid h_t)
       \prod_{t=0}^{T-1}P_{\mathrm{env}}(o_{t+1}\mid h_t,a_t).
\]
Thus the token and macro descriptions induce the same trajectory law.
They also assign the same return to each trajectory: $\gamma=1$, the
reward is unchanged, and the same stopping rule determines termination.
Their expected task returns are consequently equal.

For a fixed complete output of positive probability, the finite-product
factorization gives
\[
 \log\pi_\theta(a_t\mid h_t)
 =\sum_{i=1}^{L_t}\log\pi_\theta(a_{t,i}\mid h_{t,i}).
\]
Differentiating this finite sum proves Eq.~\eqref{eq:macro-score}, wherever
the token probabilities are differentiable. For variable-length outputs,
the stopping decision is one of these factors; when differentiating a
fixed output its length is fixed. Deterministic forced stopping
contributes a factor of one and a zero score.
\end{proof}

\begin{proof}[Proof of Lemma~\ref{lem:macro-pdl}]

For each supplied history, the values in
Eq.~\eqref{eq:macro-exact-values} use the old policy for all subsequent
decisions. They are defined throughout the histories relevant to either
policy, not merely at the histories observed in a sampled batch.
At termination let $h_{T+1}$ be an absorbing history with value zero.
The reward law and the law of $h_{t+1}$ given $(h_t,a_t)$ are common to
both policies. Consequently the exact values satisfy
\begin{align}
 Q_t^\pi(h,a)
 &=\mathbb E\left[\sum_{i=1}^{L_t}r_{t,i}
          +V_{t+1}^\pi(h_{t+1})\,\middle|\,h_t=h,a_t=a\right],
 \label{eq:macro-bellman-q}\\
 V_t^\pi(h)&=\sum_a\pi(a\mid h)Q_t^\pi(h,a),
 \qquad
 \sum_a\pi(a\mid h)A_t^\pi(h,a)=0.
 \label{eq:macro-bellman-centering}
\end{align}
Here $\sum_i r_{t,i}$ is zero except at the actual terminal turn,
where it equals $R(\tau)$; no intermediate reward is introduced.
All expectations are assumed finite, as they are for bounded rewards.

We next relate these conditional identities to the full return.
Define the one-turn temporal difference
\[
 \delta_t^\pi=\sum_{i=1}^{L_t}r_{t,i}
              +V_{t+1}^\pi(h_{t+1})-V_t^\pi(h_t).
\]
Although actions are sampled under $\pi'$, conditioning on the current
history and complete output fixes the environment's next-step law.
Equation~\eqref{eq:macro-bellman-q} therefore gives
\begin{equation}
 \mathbb E_{\pi'}[\delta_t^\pi\mid h_t,a_t]
          =Q_t^\pi(h_t,a_t)-V_t^\pi(h_t)=A_t^\pi(h_t,a_t).
 \label{eq:macro-conditional-td}
\end{equation}
For each realized trajectory the value terms telescope, since $\gamma=1$:
\[
 \sum_{t=0}^{T}\delta_t^\pi
   =\sum_{t=0}^{T}\sum_{i=1}^{L_t}r_{t,i}
          +V_{T+1}^\pi(h_{T+1})-V_0^\pi(h_0)
   =R(\tau)-V_0^\pi(h_0).
\]
Under a trajectory sampled from $\pi'$, the marginal law of $(h_t,a_t)$
is $d_t^{\pi'}(h_t)\pi'(a_t\mid h_t)$. Thus the conditional identity in
Eq.~\eqref{eq:macro-conditional-td} gives the following chain:
\begin{align*}
 \sum_{t=0}^{T}\mathbb E_{d_t^{\pi'},\pi'}[A_t^\pi]
 &=\sum_{t=0}^{T}\mathbb E_{\tau\sim\pi'}
   \left[\mathbb E_{\pi'}[\delta_t^\pi\mid h_t,a_t]\right]\\
 &=\sum_{t=0}^{T}\mathbb E_{\tau\sim\pi'}[\delta_t^\pi]\\
 &=\mathbb E_{\tau\sim\pi'}\left[\sum_{t=0}^{T}\delta_t^\pi\right]\\
 &=\mathbb E_{\tau\sim\pi'}[R(\tau)-V_0^\pi(h_0)]\\
 &=\mathcal J(\pi')-\mathbb E_{h_0\sim d_0^{\pi'}}[V_0^\pi(h_0)].
\end{align*}
The second equality is the tower property; the third exchanges a finite
sum and expectation; the fourth uses the pathwise telescoping identity.

Both policies start from the same initial-history distribution, so
$d_0^{\pi'}=d_0^\pi$. By the definition of the exact old-policy value,
$V_0^\pi(h_0)=\mathbb E_{\tau\sim\pi}[R(\tau)\mid h_0]$. Hence
\begin{align*}
 \mathbb E_{h_0\sim d_0^{\pi'}}[V_0^\pi(h_0)]
 &=\mathbb E_{h_0\sim d_0^\pi}
   \left[\mathbb E_{\tau\sim\pi}[R(\tau)\mid h_0]\right]\\
 &=\mathbb E_{\tau\sim\pi}[R(\tau)]
 =\mathcal J(\pi).
\end{align*}
Only the initial-history distributions are identified here; the full
trajectory laws under $\pi'$ and $\pi$ are not interchanged. Substituting
this equality into the preceding chain yields
\[
 \sum_{t=0}^{T}\mathbb E_{d_t^{\pi'},\pi'}[A_t^\pi]
 =\mathcal J(\pi')-\mathcal J(\pi).
\]
This proves the PDL in Eq.~\eqref{eq:macro-pdl}; it does not replace
any new-policy state distribution by an old-policy one.

It remains to establish the importance-ratio representation.
Replacing only $d_t^{\pi'}$ by $d_t^\pi$ in the PDL defines
$\mathcal S_M$; the resulting approximation error is retained below.
For its ratio representation, common support gives, at a fixed $h$,
\begin{align*}
 \sum_a\pi(a\mid h)
       \left(\frac{\pi'(a\mid h)}{\pi(a\mid h)}-1\right)A_t^\pi(h,a)
 &=\sum_a\pi'(a\mid h)A_t^\pi(h,a)
       -\sum_a\pi(a\mid h)A_t^\pi(h,a)\\
 &=\sum_a\pi'(a\mid h)A_t^\pi(h,a).
\end{align*}
Average over $d_t^\pi$ and sum over $t$ to obtain
the importance-ratio representation in Lemma~\ref{lem:macro-pdl}. Absolute continuity of the
new action law with respect to the old one suffices for this step;
the preceding PDL and the TV argument below do not require it.

\end{proof}

\subsection{Turn-Level Policy Improvement}
\label{app:macro-bound-proof}

\begin{proof}[Proof of Theorem~\ref{thm:macro-improvement}]
Sums below are over complete
outputs and token-sequence histories; the same argument uses integrals
for general observation spaces.

We first bound the difference between the turn-start occupancy measures.
Given a history $h$ and a complete output $a$, the next observation is
sampled from the same environment distribution
$P_{\mathrm{env}}(o\mid h,a)$ under both policies. The next history is
obtained by appending $a$ and $o$ to $h$. Thus the joint probabilities of
$(h,a,o)$ are $d_t^{\pi'}(h)\pi'(a\mid h)P_{\mathrm{env}}(o\mid h,a)$
and $d_t^\pi(h)\pi(a\mid h)P_{\mathrm{env}}(o\mid h,a)$, respectively.
Summing these probabilities over triples that produce the same next
history and applying the triangle inequality gives, for $0\le t<T$,
\begin{align}
 D_{\mathrm{TV}}(d_{t+1}^{\pi'},d_{t+1}^\pi)
 &\le\frac12\sum_{h,a,o}
    \left|d_t^{\pi'}(h)\pi'(a\mid h)-d_t^\pi(h)\pi(a\mid h)\right|
    P_{\mathrm{env}}(o\mid h,a)\notag\\
 &=\frac12\sum_{h,a}
    \left|d_t^{\pi'}(h)\pi'(a\mid h)-d_t^\pi(h)\pi(a\mid h)\right|\notag\\
 &\le\frac12\sum_{h,a}d_t^{\pi'}(h)
                 |\pi'(a\mid h)-\pi(a\mid h)|
    +\frac12\sum_{h,a}|d_t^{\pi'}(h)-d_t^\pi(h)|\pi(a\mid h)\notag\\
 &=\sum_h d_t^{\pi'}(h)
      D_{\mathrm{TV}}(\pi'(\cdot\mid h),\pi(\cdot\mid h))
      +D_{\mathrm{TV}}(d_t^{\pi'},d_t^\pi)\notag\\
 &\le\alpha_M+D_{\mathrm{TV}}(d_t^{\pi'},d_t^\pi).
 \label{eq:macro-detailed-state-recursion}
\end{align}
The first equality uses $\sum_oP_{\mathrm{env}}(o\mid h,a)=1$.
The second inequality adds and subtracts
$d_t^{\pi'}(h)\pi(a\mid h)$ inside the absolute value. The resulting
terms separate the change in the action distribution from the existing
difference in occupancy measures. Since $d_0^{\pi'}=d_0^\pi$, induction yields
\begin{equation}
 D_{\mathrm{TV}}(d_t^{\pi'},d_t^\pi)\le t\alpha_M,
 \qquad 0\le t\le T.
 \label{eq:macro-occupancy-tv}
\end{equation}
In particular, the first turn contributes no occupancy-measure error.

Next, the exact old-policy advantage has zero mean under $\pi$.
Using Eq.~\eqref{eq:macro-bellman-centering}, for every history $h$,
\begin{align}
 \left|\mathbb E_{a\sim\pi'(\cdot\mid h)}[A_t^\pi(h,a)]\right|
 &=\left|\sum_a[\pi'(a\mid h)-\pi(a\mid h)]A_t^\pi(h,a)\right|\notag\\
 &\le A_{\max}\sum_a|\pi'(a\mid h)-\pi(a\mid h)|
 \le2A_{\max}\alpha_M.
 \label{eq:macro-detailed-advantage-bound}
\end{align}

Finally, subtracting the surrogate from the performance-difference
identity in Lemma~\ref{lem:macro-pdl} gives the exact relation
\[
 \mathcal J(\pi')-\mathcal J(\pi)-\mathcal S_M
 =\sum_{t=0}^{T}\sum_h[d_t^{\pi'}(h)-d_t^\pi(h)]
       \mathbb E_{a\sim\pi'(\cdot\mid h)}[A_t^\pi(h,a)].
\]
Taking absolute values and applying
Eqs.~\eqref{eq:macro-detailed-advantage-bound}
and~\eqref{eq:macro-occupancy-tv}, in that order, gives
\begin{align}
 |\mathcal J(\pi')-\mathcal J(\pi)-\mathcal S_M|
 &\le\sum_{t=0}^{T}\sum_h|d_t^{\pi'}(h)-d_t^\pi(h)|
       \left|\mathbb E_{a\sim\pi'(\cdot\mid h)}[A_t^\pi(h,a)]\right|\notag\\
 &\le2A_{\max}\alpha_M
       \sum_{t=0}^{T}\sum_h|d_t^{\pi'}(h)-d_t^\pi(h)|\notag\\
 &=4A_{\max}\alpha_M
       \sum_{t=0}^{T}D_{\mathrm{TV}}(d_t^{\pi'},d_t^\pi)\notag\\
 &\le4A_{\max}\alpha_M^2\sum_{t=0}^{T}t
 =2A_{\max}T(T+1)\alpha_M^2.
 \label{eq:macro-detailed-remainder}
\end{align}
The equality uses $D_{\mathrm{TV}}(p,q)=\frac12\|p-q\|_1$,
and the last step uses $\sum_{t=0}^{T}t=T(T+1)/2$.
Taking the lower side of this absolute-error bound gives
Theorem~\ref{thm:macro-improvement}.

\end{proof}

For $T=0$, the initial occupancy measures coincide and the remainder is zero.
For $\alpha_M=0$, both the surrogate gain and the remainder are zero.
The condition in the theorem is sufficient for improvement; it
is not necessary.

\subsection{Within-Turn Surrogate Shift and KL Conversion}
\label{app:within-turn-proof}

\begin{proof}[Proof of Theorem~\ref{thm:scale-penalties}]
Fix a turn $t$, its starting history $h_t=h$, and a complete output $a_t$.
At an active token prefix $x$, write
$V^{\pi,\mathrm{tok}}(x)=\mathbb E_\pi[R(\tau)\mid h_{t,i}=x]$
for the expected return from continuing $\pi$. Thus
$V^{\pi,\mathrm{tok}}(h_{t,1})=V_t^\pi(h)$.
The exact token advantage is the expected return after fixing the next
token, minus this value. For $i<L_t$, appending $a_{t,i}$ deterministically
gives $h_{t,i+1}$, with no reward or environment interaction. Hence
\[
 A_{t,i}^{\pi,\mathrm{tok}}(h_{t,i},a_{t,i})
 =V^{\pi,\mathrm{tok}}(h_{t,i+1})-V^{\pi,\mathrm{tok}}(h_{t,i}).
\]
At the last token, the complete macro action has been specified, so
\begin{align*}
 A_{t,L_t}^{\pi,\mathrm{tok}}(h_{t,L_t},a_{t,L_t})
 &=\mathbb E\!\left[\sum_{i=1}^{L_t}r_{t,i}
             +V_{t+1}^\pi(h_{t+1})\,\middle|\,h_t,a_t\right]
             -V^{\pi,\mathrm{tok}}(h_{t,L_t})\\
 &=Q_t^\pi(h_t,a_t)-V^{\pi,\mathrm{tok}}(h_{t,L_t}).
\end{align*}
The conditional expectation averages over the environment response,
and the second equality is Eq.~\eqref{eq:macro-bellman-q}.
Summing over the tokens cancels each internal value:
\begin{equation}
 \begin{aligned}
 \sum_{i=1}^{L_t}A_{t,i}^{\pi,\mathrm{tok}}(h_{t,i},a_{t,i})
 &=\sum_{i=1}^{L_t-1}
     \bigl[V^{\pi,\mathrm{tok}}(h_{t,i+1})-V^{\pi,\mathrm{tok}}(h_{t,i})\bigr]
     +Q_t^\pi(h_t,a_t)-V^{\pi,\mathrm{tok}}(h_{t,L_t})\\
 &=Q_t^\pi(h_t,a_t)-V_t^\pi(h_t)=A_t^\pi(h_t,a_t).
 \end{aligned}
 \label{eq:exact-advantage-telescope}
\end{equation}
This identity holds for every complete output.

We next bound the difference between the within-turn occupancy measures.
Pad a stopped output to $L_{\max}$ positions with deterministic
zero-advantage actions. Extend the occupancy notation by writing
$d_{t,i}^{\pi}(\cdot\mid h)$ for the distribution of the token prefix
$h_{t,i}$ when the turn starts at $h$ and its tokens are sampled under
$\pi$; define $d_{t,i}^{\pi'}(\cdot\mid h)$ in the same way.
These are conditional occupancy measures within a fixed turn: both
start from the same $h$, regardless of how that history was reached.
In the sums below, $x$ denotes a token prefix and $a$ a single token.

Each next prefix is obtained by sampling one token and appending it.
As in Eq.~\eqref{eq:macro-detailed-state-recursion}, adding and
subtracting $d_{t,i}^{\pi'}(x\mid h)\pi(a\mid x)$ gives
\begin{align*}
 D_{\mathrm{TV}}\bigl(d_{t,i+1}^{\pi'}(\cdot\mid h),d_{t,i+1}^{\pi}(\cdot\mid h)\bigr)
 &\le\frac12\sum_{x,a}
   \left|d_{t,i}^{\pi'}(x\mid h)\pi'(a\mid x)
         -d_{t,i}^{\pi}(x\mid h)\pi(a\mid x)\right|\\
 &\le\frac12\sum_{x,a}d_{t,i}^{\pi'}(x\mid h)
             |\pi'(a\mid x)-\pi(a\mid x)|\\
 &\quad+\frac12\sum_{x,a}|d_{t,i}^{\pi'}(x\mid h)-d_{t,i}^{\pi}(x\mid h)|
             \pi(a\mid x)\\
 &=\sum_x d_{t,i}^{\pi'}(x\mid h)
       D_{\mathrm{TV}}\bigl(\pi'(\cdot\mid x),\pi(\cdot\mid x)\bigr)
       \\
 &\quad+D_{\mathrm{TV}}\bigl(d_{t,i}^{\pi'}(\cdot\mid h),d_{t,i}^{\pi}(\cdot\mid h)\bigr)\\
 &\le\alpha_T+
       D_{\mathrm{TV}}\bigl(d_{t,i}^{\pi'}(\cdot\mid h),d_{t,i}^{\pi}(\cdot\mid h)\bigr).
\end{align*}
The equality uses $\sum_a\pi(a\mid x)=1$ and the definition of total
variation. At $i=1$, both occupancy measures concentrate on $h$, so
their difference is zero. Iterating the bound $i-1$ times gives
\begin{equation}
 D_{\mathrm{TV}}\bigl(d_{t,i}^{\pi'}(\cdot\mid h),d_{t,i}^{\pi}(\cdot\mid h)\bigr)
 \le(i-1)\alpha_T.
 \label{eq:token-occupancy-tv}
\end{equation}
Padded positions have identical deterministic transitions and also
satisfy the one-step bound.

We next bound the mean old-policy token advantage under $\pi'$.
At each prefix $x$, the exact advantage satisfies
$\sum_a\pi(a\mid x)A_{t,i}^{\pi,\mathrm{tok}}(x,a)=0$;
this also holds at padded positions, where the advantage is zero.
Consequently,
\begin{align}
 \left|\mathbb E_{a\sim\pi'(\cdot\mid x)}
      [A_{t,i}^{\pi,\mathrm{tok}}(x,a)]\right|
 &=\left|\sum_a[\pi'(a\mid x)-\pi(a\mid x)]
                    A_{t,i}^{\pi,\mathrm{tok}}(x,a)\right|\notag\\
 &\le\sum_a|\pi'(a\mid x)-\pi(a\mid x)|
                   |A_{t,i}^{\pi,\mathrm{tok}}(x,a)|\notag\\
 &\le A_{\max}^{\mathrm{tok}}\sum_a|\pi'(a\mid x)-\pi(a\mid x)|\notag\\
 &=2A_{\max}^{\mathrm{tok}}D_{\mathrm{TV}}\bigl(\pi'(\cdot\mid x),\pi(\cdot\mid x)\bigr)\notag\\
 &\le2A_{\max}^{\mathrm{tok}}\alpha_T.
 \label{eq:token-mean-advantage-bound}
\end{align}
The first equality subtracts the zero mean under $\pi$. The first
inequality is the triangle inequality, and the second uses the uniform
bound $|A_{t,i}^{\pi,\mathrm{tok}}(x,a)|\le A_{\max}^{\mathrm{tok}}$.
The next equality is the definition of total variation; the final
inequality uses $\alpha_T$ as its supremum over all token prefixes.

Finally, express both surrogates using these conditional occupancy
measures. By the definition of $\mathcal S_M$ and
Eq.~\eqref{eq:exact-advantage-telescope},
\begin{align*}
 \mathcal S_M
 &=\sum_{t=0}^{T}\mathbb E_{h\sim d_t^\pi}
       \mathbb E_{a_t\sim\pi'(\cdot\mid h)}[A_t^\pi(h,a_t)]\\
 &=\sum_{t=0}^{T}\mathbb E_{h\sim d_t^\pi}
       \mathbb E_{a_t\sim\pi'(\cdot\mid h)}
       \left[\sum_{i=1}^{L_{\max}}
          A_{t,i}^{\pi,\mathrm{tok}}(h_{t,i},a_{t,i})\right]\\
 &=\sum_{t=0}^{T}\mathbb E_{h\sim d_t^\pi}
       \sum_{i=1}^{L_{\max}}\sum_x d_{t,i}^{\pi'}(x\mid h)
       \mathbb E_{a\sim\pi'(\cdot\mid x)}
          [A_{t,i}^{\pi,\mathrm{tok}}(x,a)].
\end{align*}
The last equality conditions each summand on its prefix: under $\pi'$,
the prefix has occupancy measure $d_{t,i}^{\pi'}(\cdot\mid h)$ and
the next token has distribution $\pi'(\cdot\mid x)$.
For $\mathcal S_T$, prefixes are sampled under $\pi$. At a fixed prefix,
the importance ratio changes only the distribution of the next token:
\begin{align*}
 \mathbb E_{a\sim\pi(\cdot\mid x)}
       [(\rho_{t,i}-1)A_{t,i}^{\pi,\mathrm{tok}}(x,a)]
 &=\sum_a\pi(a\mid x)
       \left(\frac{\pi'(a\mid x)}{\pi(a\mid x)}-1\right)
       A_{t,i}^{\pi,\mathrm{tok}}(x,a)\\
 &=\sum_a[\pi'(a\mid x)-\pi(a\mid x)]A_{t,i}^{\pi,\mathrm{tok}}(x,a)\\
 &=\mathbb E_{a\sim\pi'(\cdot\mid x)}[A_{t,i}^{\pi,\mathrm{tok}}(x,a)].
\end{align*}
The ratio is well defined by the support assumption, and the last
equality uses the zero old-policy mean. Conditioning $\mathcal S_T$
on the turn-start history and then on the token prefix gives
\begin{align*}
 \mathcal S_T
 &=\sum_{t=0}^{T}\mathbb E_{h\sim d_t^\pi}
       \sum_{i=1}^{L_{\max}}\sum_x d_{t,i}^{\pi}(x\mid h)
       \mathbb E_{a\sim\pi(\cdot\mid x)}
          [(\rho_{t,i}-1)A_{t,i}^{\pi,\mathrm{tok}}(x,a)]\\
 &=\sum_{t=0}^{T}\mathbb E_{h\sim d_t^\pi}
       \sum_{i=1}^{L_{\max}}\sum_x d_{t,i}^{\pi}(x\mid h)
       \mathbb E_{a\sim\pi'(\cdot\mid x)}
          [A_{t,i}^{\pi,\mathrm{tok}}(x,a)].
\end{align*}
Both expressions average turn-start histories under $d_t^\pi$; only
the conditional within-turn occupancy measures differ. Subtracting gives
\[
 \mathcal S_M-\mathcal S_T
 =\sum_{t=0}^{T}\mathbb E_{h\sim d_t^\pi}
       \sum_{i=1}^{L_{\max}}\sum_x
       [d_{t,i}^{\pi'}(x\mid h)-d_{t,i}^{\pi}(x\mid h)]
       \mathbb E_{a\sim\pi'(\cdot\mid x)}
          [A_{t,i}^{\pi,\mathrm{tok}}(x,a)].
\]
Taking absolute values, applying
Eq.~\eqref{eq:token-mean-advantage-bound}, and then using
Eq.~\eqref{eq:token-occupancy-tv} yields
\begin{align*}
 |\mathcal S_M-\mathcal S_T|
 &\le\sum_{t=0}^{T}\mathbb E_{h\sim d_t^\pi}
       \sum_{i=1}^{L_{\max}}\sum_x
       |d_{t,i}^{\pi'}(x\mid h)-d_{t,i}^{\pi}(x\mid h)|
       \left|\mathbb E_{a\sim\pi'(\cdot\mid x)}
          [A_{t,i}^{\pi,\mathrm{tok}}(x,a)]\right|\\
 &\le2A_{\max}^{\mathrm{tok}}\alpha_T
       \sum_{t=0}^{T}\mathbb E_{h\sim d_t^\pi}
       \sum_{i=1}^{L_{\max}}\sum_x
       |d_{t,i}^{\pi'}(x\mid h)-d_{t,i}^{\pi}(x\mid h)|\\
 &=4A_{\max}^{\mathrm{tok}}\alpha_T
       \sum_{t=0}^{T}\mathbb E_{h\sim d_t^\pi}
       \sum_{i=1}^{L_{\max}}
       D_{\mathrm{TV}}\bigl(d_{t,i}^{\pi'}(\cdot\mid h),d_{t,i}^{\pi}(\cdot\mid h)\bigr)\\
 &\le4A_{\max}^{\mathrm{tok}}\alpha_T^2
       \sum_{t=0}^{T}\sum_{i=1}^{L_{\max}}(i-1)\\
 &=2(T+1)A_{\max}^{\mathrm{tok}}L_{\max}(L_{\max}-1)\alpha_T^2.
\end{align*}
The first equality uses $\sum_x|p(x)-q(x)|=2D_{\mathrm{TV}}(p,q)$;
the last uses $\sum_{i=1}^{L_{\max}}(i-1)=L_{\max}(L_{\max}-1)/2$.
This proves Eq.~\eqref{eq:within-turn-shift}. Applying
Theorem~\ref{thm:macro-improvement} and then this bound gives
\begin{align*}
 \mathcal J(\pi')-\mathcal J(\pi)
 &\ge\mathcal S_M-2A_{\max}T(T+1)\alpha_M^2\\
 &\ge\mathcal S_T-|\mathcal S_M-\mathcal S_T|
       -2A_{\max}T(T+1)\alpha_M^2\\
 &\ge\mathcal S_T-2A_{\max}T(T+1)\alpha_M^2
       -2(T+1)A_{\max}^{\mathrm{tok}}L_{\max}(L_{\max}-1)\alpha_T^2,
\end{align*}
which is Eq.~\eqref{eq:token-bridge-bound}.
\end{proof}

\begin{proof}[Proof of the KL conversion in Eq.~\eqref{eq:token-to-macro-kl}]
Fix $h$ and include the stopping decision in the complete output.
Since $L(a)\le L_{\max}$, the log ratio is a bounded-length sum over
generated tokens. Conditioning each term on its prefix gives
\begin{equation}
 \begin{aligned}
 D_{\mathrm{KL}}(\pi(\cdot\mid h)\Vert\pi'(\cdot\mid h))
 &=\mathbb E_{a\sim\pi(\cdot\mid h)}
       \sum_{i=1}^{L(a)}
       \log\frac{\pi(a_i\mid h_i)}{\pi'(a_i\mid h_i)}\\
 &=\mathbb E_{a\sim\pi(\cdot\mid h)}\sum_{i=1}^{L(a)}
       D_{\mathrm{KL}}(\pi(\cdot\mid h_i)\Vert\pi'(\cdot\mid h_i))
 \le L_{\max}\delta.
 \end{aligned}
 \label{eq:kl-chain}
\end{equation}
Here logarithms are natural and the assumed finite token KL cap makes
the required absolute continuity explicit. Taking the supremum over
$h$ and applying Pinsker's inequality gives
$\alpha_M^2\le L_{\max}\delta/2$; at token prefixes it likewise gives
$\alpha_T^2\le\delta/2$.

When $R\in[0,1]$, all exact remaining-return values lie in $[0,1]$,
so $A_{\max}=A_{\max}^{\mathrm{tok}}=1$ is valid. Writing $L=L_{\max}$ and
substituting the two TV bounds into the macro and token bounds gives
\begin{align}
 \mathcal J(\pi')-\mathcal J(\pi)&\ge\mathcal S_M-T(T+1)L\delta,
 \label{eq:macro-kl-bound}\\
 \mathcal J(\pi')-\mathcal J(\pi)&\ge\mathcal S_T-T(T+1)L\delta
                         -(T+1)L(L-1)\delta.
 \label{eq:token-kl-bound}
\end{align}
For $\delta=\Delta/L$, the last penalty is $(T+1)(L-1)\Delta$,
as claimed in the main text.
\end{proof}

\subsection{Arbitrary Estimators and Action-Independent Baselines}
\label{sec:estimator-bound}

Let $\widehat A_t$ be any integrable scalar macro estimate computed on
old-policy rollouts, with its construction fixed during the update.
Assume all displayed importance-weighted expectations are finite.
Define its conditional mean error and population estimated improvement by
\begin{align}
 b_t(h,a)&=\mathbb E_\pi[\widehat A_t\mid h_t=h,a_t=a]
                 -A_t^\pi(h,a),\label{eq:conditional-mean-error}\\
 \widehat{\mathcal S}_M
   &=\mathbb E_\pi\sum_{t=0}^{T}(\rho_t-1)\widehat A_t.
   \label{eq:estimated-macro-surrogate}
\end{align}
The expectation includes future rollout randomness. The hat denotes
an estimated advantage, not an empirical average.

\begin{theorem}[Estimator-robust policy improvement]
\label{thm:estimator-improvement}
Under Theorem~\ref{thm:macro-improvement}, suppose state-only functions
$c_t$ satisfy $|b_t(h,a)-c_t(h)|\le\epsilon_{A,t}$ uniformly.
Then
\begin{equation}
 \begin{aligned}
 \mathcal J(\pi')-\mathcal J(\pi)\ge{}&\widehat{\mathcal S}_M
       -2A_{\max}T(T+1)\alpha_M^2\\
       &-2\alpha_M\sum_{t=0}^{T}\epsilon_{A,t}.
 \end{aligned}
 \label{eq:estimator-improvement}
\end{equation}
In particular, a common bound $\epsilon_A$ gives an estimation penalty
$2(T+1)\alpha_M\epsilon_A$.
\end{theorem}

\begin{proof}[Proof of Theorem~\ref{thm:estimator-improvement}]
The estimator is computed under the old-policy rollout law, and
$\rho_t-1$ is a function of $(h_t,a_t)$. Conditioning on that pair gives
\begin{align*}
 \widehat{\mathcal S}_M-\mathcal S_M
 &=\sum_{t=0}^{T}\mathbb E_\pi
       [(\rho_t-1)(\widehat A_t-A_t^\pi(h_t,a_t))]\\
 &=\sum_{t=0}^{T}\mathbb E_{h\sim d_t^\pi}
       \sum_a\pi(a\mid h)
       \left(\frac{\pi'(a\mid h)}{\pi(a\mid h)}-1\right)b_t(h,a)\\
 &=\sum_{t=0}^{T}\mathbb E_{h\sim d_t^\pi}
       \sum_a[\pi'(a\mid h)-\pi(a\mid h)]b_t(h,a).
\end{align*}
The integrability assumptions justify these conditional expectations.
For each fixed $h$, the state-only part cancels because
\[
 \sum_a[\pi'(a\mid h)-\pi(a\mid h)]c_t(h)
 =c_t(h)(1-1)=0.
\]
Subtracting this zero term before taking absolute values yields
\begin{align*}
 |\widehat{\mathcal S}_M-\mathcal S_M|
 &\le\sum_{t=0}^{T}\mathbb E_{h\sim d_t^\pi}
       \sum_a|\pi'(a\mid h)-\pi(a\mid h)|\,|b_t(h,a)-c_t(h)|\\
 &\le\sum_{t=0}^{T}2\alpha_M\epsilon_{A,t}.
\end{align*}
Therefore $\mathcal S_M\ge\widehat{\mathcal S}_M
-2\alpha_M\sum_t\epsilon_{A,t}$. Substitution into
Eq.~\eqref{eq:exact-macro-bound} proves the theorem. If all
$\epsilon_{A,t}\le\epsilon_A$, summing $T+1$ terms gives the stated
uniform-error penalty.
\end{proof}

\subsection{Boundary Critic Error and the Turn-Credit Guarantee}
\label{app:turn-credit-error}

For the turn GAE in Eq.~\eqref{eq:turn-gae}, write
$e_t(h_t)=V_{t,1}-V_t^\pi(h_t)$ for the boundary critic error.

\begin{corollary}[Policy improvement with turn-level GAE]
\label{cor:macro-gae-improvement}
Under the assumptions of Theorem~\ref{thm:estimator-improvement},
suppose rollouts reach true termination, $e_{T+1}=0$,
$\lambda\in[0,1]$, and $|e_t(h)|\le\epsilon_B$. Then one may take
\begin{equation}
 \epsilon_{A,t}=\epsilon_B(1-\lambda^{T-t}),\qquad
 \epsilon_A=\epsilon_B(1-\lambda^{T}).
 \label{eq:macro-gae-error}
\end{equation}
Thus, with $\widehat{\mathcal S}_M$ formed from $A_t^{\mathrm{turn}}$,
\begin{equation}
 \begin{aligned}
 \mathcal J(\pi')-\mathcal J(\pi)\ge{}&\widehat{\mathcal S}_M
           -2A_{\max}T(T+1)\alpha_M^2\\
           &-2(T+1)\alpha_M\epsilon_B(1-\lambda^{T}).
 \end{aligned}
 \label{eq:macro-gae-improvement}
\end{equation}
\end{corollary}

\begin{proof}[Proof of Corollary~\ref{cor:macro-gae-improvement}]
Use the exact-critic residual $\delta_s^\pi$ defined in the proof of
Lemma~\ref{lem:macro-pdl}. First fix $(h_t,a_t)=(h,a)$ and continue
with $\pi$. Equation~\eqref{eq:macro-conditional-td} gives
$\mathbb E_\pi[\delta_t^\pi\mid h_t=h,a_t=a]=A_t^\pi(h,a)$.
For every $s>t$, the full history $h_s$ contains $(h_t,a_t)$, and
\[
 \mathbb E_\pi[\delta_s^\pi\mid h_s]
 =\sum_{a_s}\pi(a_s\mid h_s)A_s^\pi(h_s,a_s)=0.
\]
The tower property consequently gives
$\mathbb E_\pi[\delta_s^\pi\mid h_t=h,a_t=a]=0$ for $s>t$.
It follows that the entire exact-critic trace has conditional mean
\[
 \mathbb E_\pi\!\left[\sum_{\ell=0}^{T-t}\lambda^\ell
           \delta_{t+\ell}^\pi\,\middle|\,h_t=h,a_t=a\right]
 =A_t^\pi(h,a).
\]

Replacing exact boundary values by critic predictions changes the
residual at turn $s$ by $e_{s+1}(h_{s+1})-e_s(h_s)$.
Suppress only the history arguments in the following pathwise sum:
\begin{align*}
 A_t^{\mathrm{turn}}-\sum_{\ell=0}^{T-t}\lambda^\ell\delta_{t+\ell}^\pi
 &=\sum_{\ell=0}^{T-t}\lambda^\ell(e_{t+\ell+1}-e_{t+\ell})\\
 &=\sum_{j=1}^{T-t+1}\lambda^{j-1}e_{t+j}
        -e_t-\sum_{j=1}^{T-t}\lambda^je_{t+j}\\
 &=-e_t+(1-\lambda)\sum_{j=1}^{T-t}\lambda^{j-1}e_{t+j}
        +\lambda^{T-t}e_{T+1}.
\end{align*}
True termination makes the last term zero. Taking the conditional
expectation and subtracting the exact advantage therefore gives
\begin{equation}
 b_t(h,a)=-e_t(h)+(1-\lambda)\sum_{j=1}^{T-t}\lambda^{j-1}
          \mathbb E_\pi[e_{t+j}(h_{t+j})\mid h_t=h,a_t=a].
 \label{eq:macro-gae-error-identity}
\end{equation}
Choose the state-only baseline $c_t(h)=-e_t(h)$. For $0\le\lambda<1$,
\[
 |b_t(h,a)-c_t(h)|
 \le(1-\lambda)\sum_{j=1}^{T-t}\lambda^{j-1}\epsilon_B
 =\epsilon_B(1-\lambda^{T-t}).
\]
For $\lambda=1$, the finite sum in
Eq.~\eqref{eq:macro-gae-error-identity} is multiplied by zero, so the
same bound holds with value zero. Empty sums vanish and $\lambda^0=1$,
including when $\lambda=0$; thus the bound also covers $t=T$ and $T=0$.

Theorem~\ref{thm:estimator-improvement} now gives an estimation penalty
$2\alpha_M\epsilon_B\sum_{t=0}^{T}(1-\lambda^{T-t})$.
Since $1-\lambda^{T-t}\le1-\lambda^T$, this is at most
$2(T+1)\alpha_M\epsilon_B(1-\lambda^T)$, proving
Eqs.~\eqref{eq:macro-gae-error} and~\eqref{eq:macro-gae-improvement}.
\end{proof}

The cancellation of $e_{T+1}$ uses true termination. At a nonterminal
truncation, that term must be retained unless its bootstrap error is
separately controlled.

\subsection{Token GAE and Internal Critic Error}
\label{sec:gae}

For comparison, standard token GAE applies trace decay at every generated
assistant token:
\begin{equation}
 A_{t,i}^{\mathrm{GAE}}(\lambda)
 =\delta_{t,i}+\lambda A_{t,i+1}^{\mathrm{GAE}}(\lambda).
 \label{eq:gae}
\end{equation}
The trace continues across turns using
$A_{t,L_t+1}^{\mathrm{GAE}}(\lambda):=A_{t+1,1}^{\mathrm{GAE}}(\lambda)$,
with $A_{T+1,1}^{\mathrm{GAE}}(\lambda):=0$ at termination.
Ordering assistant actions by $k(t,i)=\sum_{u<t}L_u+i$ gives
$N=\sum_{t=0}^{T}L_t$ action positions. In this flattened notation,
\begin{equation}
 A_k^{\mathrm{GAE}}(\lambda)=\sum_{j=k}^{N}\lambda^{j-k}\delta_j.
 \label{eq:gae-unrolled}
\end{equation}
Environment tokens introduce neither TD residuals nor decay steps, but
credit across a turn still decays at each assistant token.
At $\lambda=1$, the value differences telescope to
\begin{equation}
 A_k^{\mathrm{GAE}}(1)=R(\tau)-V_k.
 \label{eq:gae-telescoping}
\end{equation}

To compare trace lengths, specialize to exactly $T+1$ turns of $L$ tokens
each, with $N=(T+1)L$. At flattened position $j=k(t,i)$, let
$e_j^T(h)$ be the critic prediction minus the exact old-policy token value
at history $h$, and set $e_{N+1}^T=0$. All critic predictions are frozen
during the policy update. Inside expectations, errors are evaluated at
the corresponding random histories.

\begin{proposition}[Token GAE error and policy improvement]
\label{prop:token-gae-error}
Under the assumptions of Theorem~\ref{thm:scale-penalties}, suppose
rollouts reach true termination, $q\in[0,1]$, and $|e_j^T(h)|\le\epsilon_T$.
The conditional bias of token GAE relative to the exact token advantage is
\begin{equation}
 b_j^T(h,a)=-e_j^T(h)+(1-q)\sum_{m=1}^{N-j}q^{m-1}
 \mathbb E_\pi[e_{j+m}^T\mid h_{t,i}=h,a_{t,i}=a].
 \label{eq:token-gae-bias}
\end{equation}
With $\widehat{\mathcal S}_T=\mathbb E_\pi\sum_{j=1}^N
(\rho_j-1)A_j^{\mathrm{GAE}}(q)$ and finite displayed expectations,
\begin{equation}
 \begin{aligned}
 \mathcal J(\pi')-\mathcal J(\pi)\ge{}&\widehat{\mathcal S}_T
 -2A_{\max}T(T+1)\alpha_M^2\\
 &-2(T+1)A_{\max}^{\mathrm{tok}}L(L-1)\alpha_T^2
 -2N\alpha_T\epsilon_T(1-q^{N-1}).
 \end{aligned}
 \label{eq:token-gae-improvement}
\end{equation}
\end{proposition}

\begin{proof}[Proof of Proposition~\ref{prop:token-gae-error}]
Write $\delta_j^\pi$ for the token TD residual formed with exact values.
The token Bellman equation gives
$\mathbb E_\pi[\delta_j^\pi\mid h_{t,i}=h,a_{t,i}=a]
=A_{t,i}^{\pi,\mathrm{tok}}(h,a)$.
At each later position $\ell>j$, the token action is sampled under $\pi$,
so its exact advantage has conditional mean zero given its prefix.
That prefix contains the fixed earlier history and action. The tower
property therefore gives
$\mathbb E_\pi[\delta_\ell^\pi\mid h_{t,i}=h,a_{t,i}=a]=0$.
Consequently the conditional mean of
$\sum_{m=0}^{N-j}q^m\delta_{j+m}^\pi$ is the exact advantage at position $j$.

The difference between the estimated and exact residuals is
$e_{j+1}^T-e_j^T$. Collecting its coefficients in the finite trace gives
\begin{align*}
 A_j^{\mathrm{GAE}}(q)-\sum_{m=0}^{N-j}q^m\delta_{j+m}^\pi
 &=\sum_{m=0}^{N-j}q^m(e_{j+m+1}^T-e_{j+m}^T)\\
 &=-e_j^T+\sum_{m=1}^{N-j}(q^{m-1}-q^m)e_{j+m}^T
       +q^{N-j}e_{N+1}^T\\
 &=-e_j^T+(1-q)\sum_{m=1}^{N-j}q^{m-1}e_{j+m}^T.
\end{align*}
Taking conditional expectations proves Eq.~\eqref{eq:token-gae-bias}.
The geometric sum bounds the remaining action-dependent part by
\[
 |b_j^T(h,a)+e_j^T(h)|\le\epsilon_T(1-q^{N-j}).
\]
For $q=1$ this follows directly from the zero factor $1-q$; for $j=N$
the sum is empty. We use $q^0=1$ also at $q=0$.

Condition on each old-policy token history in the difference of surrogates.
The importance ratio changes only the next-token distribution. Moreover,
$\sum_a[\pi'(a\mid h)-\pi(a\mid h)]e_j^T(h)=0$. Thus
\begin{align*}
 |\widehat{\mathcal S}_T-\mathcal S_T|
 &=\left|\sum_{j=1}^N\mathbb E_{h\sim d_j^{\pi,\mathrm{tok}}}
     \sum_a[\pi'(a\mid h)-\pi(a\mid h)]
                    [b_j^T(h,a)+e_j^T(h)]\right|\\
 &\le2\alpha_T\epsilon_T\sum_{j=1}^N(1-q^{N-j})
 \le2N\alpha_T\epsilon_T(1-q^{N-1}),
\end{align*}
where $d_j^{\pi,\mathrm{tok}}$ denotes the old-policy history law at
flattened position $j$. Substituting the resulting lower bound on
$\mathcal S_T$ into Eq.~\eqref{eq:token-bridge-bound} proves the claim.
\end{proof}

The conditioning in Eq.~\eqref{eq:token-gae-bias} fixes one token, not a
complete turn. This bias therefore cannot be used as the scalar macro
bias in Theorem~\ref{thm:estimator-improvement}. The bound above includes
both the within-turn distribution shift and errors at internal prefixes.

\subsection{Matched-Budget Comparison and an Internal-Error Example}
\label{app:gae-scale-comparison}

\begin{corollary}[Comparison under a common token KL budget]
\label{cor:matched-budget}
Assume the hypotheses of Corollary~\ref{cor:macro-gae-improvement} and
Proposition~\ref{prop:token-gae-error}.
As in Appendix~\ref{sec:gae}, consider $(T+1)$ turns of $L$ tokens each,
with $N=(T+1)L$. For $R\in[0,1]$, a uniform token KL cap $\delta$
gives the bounds $\mathcal J(\pi')-\mathcal J(\pi)\ge
\widehat{\mathcal S}_M-P_M$ and $\mathcal J(\pi')-\mathcal J(\pi)\ge
\widehat{\mathcal S}_T-P_T$, with the penalty envelopes
\begin{align}
 P_M={}&T(T+1)L\delta\notag\\
       &+2(T+1)\sqrt{L\delta/2}\,\epsilon_B(1-\lambda^{T}),
       \label{eq:macro-penalty-envelope}\\
 P_T={}&T(T+1)L\delta+(T+1)L(L-1)\delta\notag\\
       &+2(T+1)L\sqrt{\delta/2}\,\epsilon_T(1-q^{(T+1)L-1}).
       \label{eq:token-penalty-envelope}
\end{align}
If a common critic-error cap $\epsilon_B=\epsilon_T=\epsilon>0$ is
available, $T>0$, $\delta>0$, and $q=\lambda<1$, the token-to-macro
ratio of the estimation terms is
\begin{equation}
 \sqrt L\,\frac{1-\lambda^{(T+1)L-1}}{1-\lambda^{T}}\ge\sqrt L.
 \label{eq:estimation-scale-ratio}
\end{equation}
This also exceeds $\sqrt L$ for $L>1$ when matching boundary decay
$q^L=\lambda\in(0,1)$, since $q^{(T+1)L-1}=\lambda^{T+1-1/L}$.
At $q=\lambda=1$ both estimation terms vanish; their ratio is undefined.
\end{corollary}

\begin{proof}[Proof of Corollary~\ref{cor:matched-budget}]
Because $R\in[0,1]$, exact values and action values are in $[0,1]$,
so both macro and token advantages have magnitude at most one.
Pinsker's inequality and Eq.~\eqref{eq:token-to-macro-kl} give
$\alpha_T^2\le\delta/2$ and $\alpha_M^2\le L\delta/2$.
All penalty coefficients in the two improvement bounds are nonnegative.
Replacing the advantage bounds by one, $\alpha_T$ by $\sqrt{\delta/2}$,
and $\alpha_M$ by $\sqrt{L\delta/2}$ therefore gives valid upper bounds
on those penalties. In particular the two distribution-shift terms become
\[
 2T(T+1)\frac{L\delta}{2}=T(T+1)L\delta,\qquad
 2(T+1)L(L-1)\frac{\delta}{2}=(T+1)L(L-1)\delta.
\]
The estimation terms undergo the same substitution, yielding
Eqs.~\eqref{eq:macro-penalty-envelope}--\eqref{eq:token-penalty-envelope}.

Under the common positive error cap, $T>0$, $\delta>0$, and
$q=\lambda<1$, the macro estimation term is positive. Dividing gives
\[
 \frac{2(T+1)L\sqrt{\delta/2}\,\epsilon(1-\lambda^{(T+1)L-1})}
 {2(T+1)\sqrt{L\delta/2}\,\epsilon(1-\lambda^T)}
 =\sqrt L\frac{1-\lambda^{(T+1)L-1}}{1-\lambda^T}.
\]
Since $(T+1)L-1\ge T$ and $0\le\lambda<1$, the numerator in the final
fraction is at least its denominator. This proves
Eq.~\eqref{eq:estimation-scale-ratio}. If instead $q^L=\lambda\in(0,1)$,
the exponent becomes $T+1-1/L$, strictly greater than $T$ for $L>1$.
Thus $1-q^{(T+1)L-1}>1-\lambda^T$, proving the strict comparison.
For $q=\lambda=1$, each estimation term is zero directly from its formula.
\end{proof}

The common error cap is an explicit comparison assumption; it does not
follow merely from using one critic. These compare certified worst-case
penalties, not actual biases or achieved improvements. In particular,
the two surrogates on the right-hand sides need not be equal.

\begin{proposition}[Internal critic error can reverse the update direction]
\label{prop:internal-error-example}

Consider one turn of $L>1$ tokens. Its first token chooses a fair sign
$Y\in\{-1,1\}$ under the old policy; the remaining tokens are
deterministic. The terminal reward is $R=1/2+cY$, with $0<c<1/2$.
Use exact boundary values $V_{0,1}=1/2$, $V_{1,1}=0$,
but internal predictions $V_{0,j}=1/2+(c-\varepsilon)Y$ for
$2\le j\le L$. Each internal value error has magnitude $\varepsilon$.
The two estimates relevant to the only nonzero policy score are
\begin{equation}
 A_0^{\mathrm{turn}}=cY,\qquad
 A_1^{\mathrm{GAE}}(q)
       =[c-\varepsilon(1-q^{L-1})]Y.
 \label{eq:internal-error-counterexample}
\end{equation}
When $0<c<\varepsilon(1-q^{L-1})$, increasing the token surrogate favors
the lower-reward action, whereas the macro surrogate favors the
higher-reward action. This holds for arbitrarily small policy updates
and with both critics' predictions in $[0,1]$, for example
$c=0.05$, $\varepsilon=0.2$, $q=0.8$, $L=4$.
The objective and boundary predictions are identical; the difference
comes from bootstrapping at internal prefixes.

\end{proposition}

\begin{proof}[Proof of Proposition~\ref{prop:internal-error-example}]
Only the first and
last TD residuals are nonzero:
$\delta_1=(c-\varepsilon)Y$ and $\delta_L=\varepsilon Y$.
Their GAE sum is $(c-\varepsilon+q^{L-1}\varepsilon)Y$.
The macro residual is exactly $cY$. If the new probability of $Y=1$
is $1/2+u$, then
\[
 \mathcal J(\pi')-\mathcal J(\pi)=2cu,\qquad
 \widehat{\mathcal S}_T=2[c-\varepsilon(1-q^{L-1})]u.
\]
Choosing $u<0$ makes the second quantity positive and the first negative
under the stated condition. For the numerical values in the text and
$u=-0.1$, the surrogate increases by $0.00952$ and the return decreases
by $0.01$. The nonbaseline bias radius at the first token is
$\varepsilon(1-q^{L-1})$, and the other tokens have no policy change.
Its exact surrogate-error bound is $2|u|\varepsilon(1-q^{L-1})=0.01952$,
which accounts for this reversal. Letting $|u|$ approach zero also
satisfies arbitrarily small positive KL budgets: the only changing
conditional action law has
$D_{\mathrm{KL}}(\pi\Vert\pi')=-\tfrac12\log(1-4u^2)\to0$ as $u\to0$.
For the specified example the internal predictions are $0.35$ and $0.65$,
and $c-\varepsilon(1-q^{L-1})=-0.0476<0$, verifying the reward-range
and sign conditions explicitly.
\end{proof}

\subsection{Macro and Token Ratio Objectives}
\label{app:product-ratio-discrepancy}

Broadcasting a fixed scalar turn estimate $\widehat A_t$ gives the
token objective
\[
 \mathcal T(\widehat A)=\mathbb E_\pi\sum_t\sum_i
                  (\rho_{t,i}-1)\widehat A_t.
\]
\begin{proposition}[Local agreement of broadcast ratio objectives]
\label{prop:ratio-local-agreement}
In general, $\mathcal T(\widehat A)$ and $\widehat{\mathcal S}_M$ differ
at finite updates. Their gradients agree at the old policy when both use
the same frozen estimates and summation weights, token probabilities are
differentiable with common support, and differentiation may be interchanged
with the rollout expectation.
\end{proposition}

\begin{proof}[Proof of Proposition~\ref{prop:ratio-local-agreement}]
Let $\theta_0$ parameterize the old policy, and let $\rho_{t,i}(\theta)$
be the token ratios on a fixed old-policy rollout. The macro ratio is
$\rho_t(\theta)=\prod_i\rho_{t,i}(\theta)$, and every token ratio is one
at $\theta_0$. The finite-product rule gives
\[
 \left.\nabla_\theta(\rho_t(\theta)-1)\right|_{\theta_0}
 =\left.\sum_i\left(\prod_{j\ne i}\rho_{t,j}(\theta)\right)
                       \nabla_\theta\rho_{t,i}(\theta)\right|_{\theta_0}
 =\left.\nabla_\theta\sum_i(\rho_{t,i}(\theta)-1)\right|_{\theta_0}.
\]
Multiplying by the frozen $\widehat A_t$, summing over turns, and taking
the fixed old-policy expectation proves gradient equality under the stated
interchange assumption (for example, integrably dominated gradients suffice).

For finite updates the omitted product terms need not cancel even in
expectation. Consider one turn with two independent fair signs $X,Y$ under
$\pi$, a frozen estimate $\widehat A=XY$, and new probabilities
$\pi'(X=1)=1/2+u$, $\pi'(Y=1)=1/2+v$, where $|u|,|v|<1/2$.
Then $\rho_1=1+2uX$ and $\rho_2=1+2vY$. Independence and zero means give
\[
 \mathcal T=\mathbb E[(2uX+2vY)XY]=0,\qquad
 \widehat{\mathcal S}_M
 =\mathbb E[(2uX+2vY+4uvXY)XY]=4uv.
\]
These differ whenever $uv\ne0$, as claimed.
\end{proof}

PPO clipping further changes the objective and does not impose a hard
macro trust region \citep{schulman2017ppo}. The proposition concerns the
unclipped objectives and does not assert finite-update equivalence.

\section{Implementation and Ablation Settings}
\label{app:mae-settings}

\paragraph{Training data.}
For BCP, we follow the training-data setup of SUPO~\citep{lu2025supo}.
The training data for Claw and Automation are proprietary internal
company datasets; their specific details cannot be disclosed.

\begin{table}[!htbp]
 \centering
 \caption{RL training hyperparameters for BCP, Claw, and Automation.
 The BCP settings are shared by the 4B and 30B-A3B runs. Length limits
 are measured in tokens. The \texttt{max\_turn} row corresponds to
 \texttt{max\_assistant\_turns} for BCP, \texttt{max\_turns} for Claw,
 and \texttt{max\_steps} for Automation.}
 \label{tab:rl-training-hyperparameters}
 \renewcommand{\arraystretch}{1.15}
 \begin{tabular*}{0.85\textwidth}{@{\extracolsep{\fill}}lccc@{}}
  \toprule
  Hyperparameter & BCP & Claw & Automation \\
  \midrule
  \texttt{max\_prompt\_length} & 4096 & 8000 & 8000 \\
  \texttt{max\_response\_length} & 32768 & 64000 & 64000 \\
  \texttt{train\_batch\_size} & 32 & 64 & 64 \\
  Actor \texttt{optim.lr} & $10^{-6}$ & $1.6\times10^{-6}$ & $1.6\times10^{-6}$ \\
  Critic \texttt{optim.lr} & $10^{-5}$ & $10^{-5}$ & $10^{-5}$ \\
  \texttt{rollout.n} & 8 & 8 & 8 \\
  \texttt{max\_turn} & 20 & 20 & 50 \\
  \bottomrule
 \end{tabular*}
\end{table}

\paragraph{Computing the two traces.}
The turn anchor can be computed by broadcasting $V_{t,1}$ within
each turn and running a token GAE scan with decay $1$ internally
and $\lambda$ at the final token. Internal TD errors then vanish,
and the turn-end residual equals Eq.~\eqref{eq:turn-td}.
The local branch uses the original token values and the backward recursion
\begin{equation}
 A_{t,i}^{\mathrm{intra}}=
 \begin{cases}
  \delta_{t,i}+\lambda_{\mathrm{intra}}A_{t,i+1}^{\mathrm{intra}},
    &i<L_t,\\
  \delta_{t,L_t},&i=L_t.
 \end{cases}
 \label{eq:intra-recursion}
\end{equation}
The within-turn mean used for centering is
\begin{equation}
 \overline A_t^{\mathrm{intra}}
 =\frac{1}{L_t}\sum_{i=1}^{L_t}A_{t,i}^{\mathrm{intra}},
 \label{eq:intra-mean}
\end{equation}

\paragraph{Critic target.}
We train the critic by directly regressing the final Monte Carlo return:
$\widehat R_{t,i}=R(\tau)$ at every valid assistant token.
The target is detached, and each prediction is conditioned on the
corresponding token prefix.

\begin{algorithm}[htbp]
 \caption{Multiscale Advantage Estimation}
 \label{alg:mae}
 \begin{algorithmic}[1]
  \REQUIRE Trajectory $\tau=(o_0,a_0,\ldots,o_T,a_T)$, values $V_{t,i}$,
  terminal reward $R(\tau)$;
  $\lambda$, $\lambda_{\mathrm{intra}}$, $\beta$
  \STATE Set $V_{T+1,1}=0$ and $A_{T+1}^{\mathrm{turn}}=0$.
  \STATE Read $L_t=|a_t|$; assign $r_{t,i}$ using
  Eq.~\eqref{eq:terminal-reward}.
  \STATE Compute token TD errors $\delta_{t,i}$ using
  Eqs.~\eqref{eq:value-boundary} and~\eqref{eq:td}.
  \FOR{$t=T,\ldots,0$}
   \STATE Compute $\delta_t^{\mathrm{turn}}$ and $A_t^{\mathrm{turn}}$
   using Eqs.~\eqref{eq:turn-td}--\eqref{eq:turn-gae}.
   \STATE Compute $A_{t,i}^{\mathrm{intra}}$ for $i=L_t,\ldots,1$ using
   Eq.~\eqref{eq:intra-recursion}.
   \STATE Center and combine using Eqs.~\eqref{eq:intra-mean} and~\eqref{eq:mae}.
  \ENDFOR
  \STATE Set $\widehat R_{t,i}=R(\tau)$ at every valid assistant token.
  \STATE Normalize $A_{t,i}^{\mathrm{MAE}}$ over valid batch tokens.
  \STATE \textbf{return} normalized advantages and critic targets $\widehat R_{t,i}$.
 \end{algorithmic}
\end{algorithm}

\paragraph{Configuration.}
Our reference configuration uses $\lambda=0.8$,
$\lambda_{\mathrm{intra}}=0.95$, $\lambda_{\mathrm{critic}}=1$, and $\beta=1$.
We use $\gamma=1$, an explicit assistant action mask, and the terminal
reward placement in Eq.~\eqref{eq:terminal-reward}; dense token-level KL
rewards are excluded.
All traces and turn reductions take linear time in the stored trajectory
length. Algorithm~\ref{alg:mae} summarizes the estimator.

\paragraph{Advantage normalization.}
The implementation subtracts the batch mean and divides by the square
root of the sample variance plus $10^{-8}$. Masked positions are zeroed;
batches with at most one valid action receive zero policy advantages.

\paragraph{Parameter ablations.}
The 4B Search ablations in Tables~\ref{tab:ablation-credit}
and~\ref{tab:ablation-beta} vary one parameter at a time around
$(\lambda,\lambda_{\mathrm{intra}},\beta)=(0.80,0.95,1.00)$,
with $\gamma=1$ and $\lambda_{\mathrm{critic}}=1$ fixed.
The evaluated settings are $\lambda\in\{0.60,0.80,0.85\}$,
$\lambda_{\mathrm{intra}}\in\{0.90,0.95,1.00\}$, and
$\beta\in\{0.0,0.5,1.0,1.5,2.0\}$.
Each configuration uses one training run. We select its best-performing
checkpoint and report the mean of five repeated offline evaluations,
following the protocol in Section~\ref{sec:experimental-settings}.
At $\beta=0$, only turn-level GAE remains and
$\lambda_{\mathrm{intra}}$ has no effect on the actor credit.

\section{Critic Values Across Turns}
\label{app:critic-observation}

Figure~\ref{fig:critic-value-population} summarizes the original
assistant-token critic predictions for all 2,048 trajectories and
16,323 turns from eight checkpoints of one MPO 4B search run.

\paragraph{Value differences are primarily between turns.}
Panel (a) compares variation around each turn's mean with variation
of turn means around the trajectory mean. Most points lie above the
equality line: between-turn variation exceeds within-turn variation
in $93.8\%$ of trajectories. Panel (b) shows that this pattern persists
across checkpoints. The median share of variance explained by
between-turn differences is $89.5\%$ when weighting by token count
and $90.2\%$ when weighting turns equally, indicating that the pattern
is not driven solely by differences in turn length.

\paragraph{Local token variation remains relevant.}
Panel (c) compares jumps from the last token value of one turn to the
first of the next with the within-turn P95--P05 value span.
The mean boundary jump exceeds the mean within-turn span in only
$22.4\%$ of trajectories. Panel (d) further shows that boundaries
account for $9.2\%$ of pooled squared adjacent value changes;
small within-turn fluctuations accumulate over many more tokens.
Thus, differences in turn-level value do not imply that most local
value movement occurs at boundaries. Together, these observations
motivate organizing long-range credit at the turn scale while
retaining token-level refinement within each turn.

\begin{figure}[htbp]
 \centering
 \includegraphics[width=0.98\textwidth]{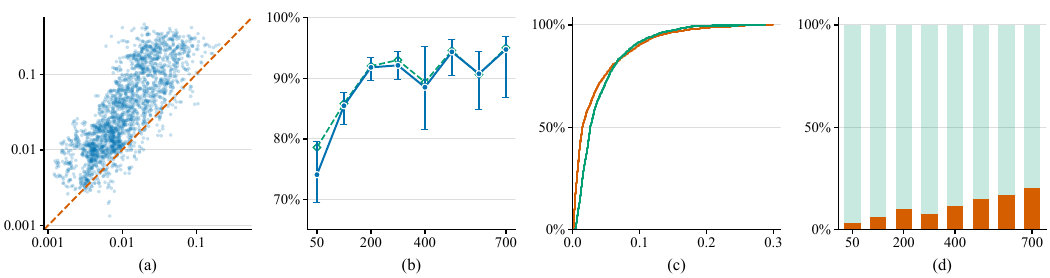}
 \caption{\textbf{Critic-value structure across turns.}
 (a) Within-turn (horizontal) versus between-turn (vertical) RMS
 deviations, weighted by token count. Each point is a trajectory;
 the dashed line marks equality.
 (b) Median between-turn variance share versus training step, with
 token weighting (blue) and equal turn weighting (green).
 Error bars are 95\% prompt-cluster bootstrap intervals within each checkpoint.
 (c) Cumulative distributions of per-trajectory mean absolute boundary
 jumps (orange) and mean within-turn P95--P05 spans (green).
 (d) Shares of pooled squared adjacent value changes at turn boundaries
 (orange) and within turns (pale green), by training step.}
 \label{fig:critic-value-population}
\end{figure}

\clearpage
\subsection{More Examples}
\label{app:critic-more-examples}

Figure~\ref{fig:critic-more-examples} shows six additional trajectories
with pronounced critic-value changes at multiple turn boundaries,
complementing the examples in Figure~\ref{fig:critic-turn-motivation}.

\begin{figure}[!htbp]
 \centering
 \includegraphics[width=\textwidth]{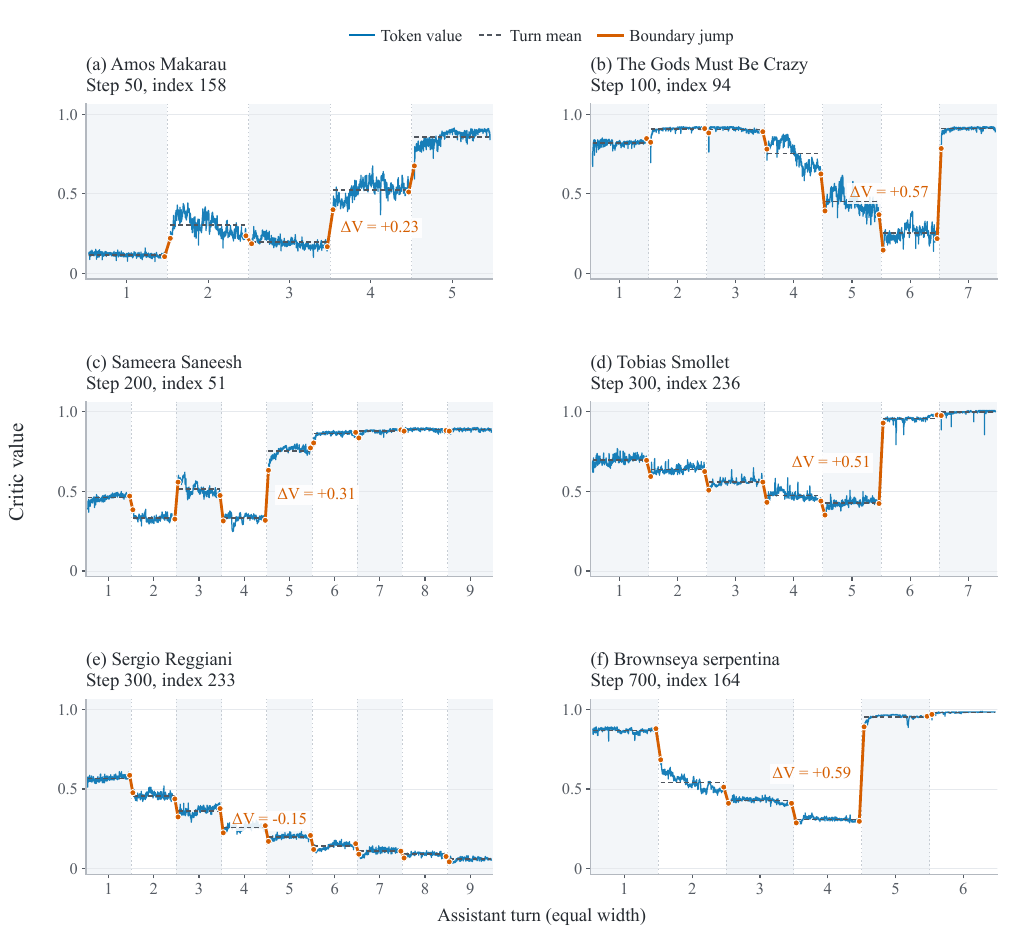}
 \caption{\textbf{More examples of critic-value structure across turns.}
 Blue curves show raw token values, dashed lines show turn means, and
 orange segments mark changes between successive turns.
 As in Figure~\ref{fig:critic-turn-motivation}, turns have
 equal display width and curves are unsmoothed; only assistant-action
 positions are included, with no appended terminal zero.}
 \label{fig:critic-more-examples}
\end{figure}

\clearpage
\section{Validation Curves for Parameter Ablations}
\label{app:ablation-validation-curves}

Figures~\ref{fig:ablation-val-beta}--\ref{fig:ablation-val-intra-lambda}
compare 4B BrowseComp-Plus validation pass@1 through training step 300.
Red curves show the MPO reference with
$(\lambda,\lambda_{\mathrm{intra}},\beta)=(0.80,0.95,1.00)$;
blue curves show one-parameter ablations. All panels use the same
axis ranges and display unsmoothed validation measurements.

\begingroup
\setlength{\intextsep}{6pt plus 1pt minus 1pt}
\setlength{\floatsep}{6pt plus 1pt minus 1pt}
\begin{figure}[!htbp]
 \centering
 \includegraphics[width=0.32\textwidth]{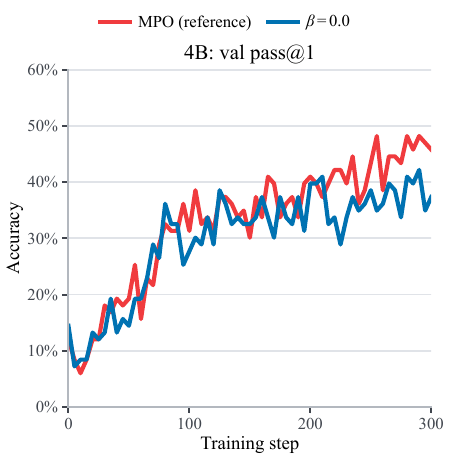}
 \hspace{0.04\textwidth}
 \includegraphics[width=0.32\textwidth]{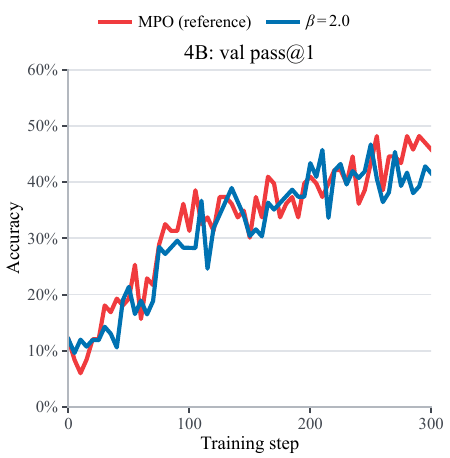}
 \caption{Local refinement strength $\beta$: 0.0 (left) and 2.0 (right),
 compared with the MPO reference.}
 \label{fig:ablation-val-beta}
\end{figure}

\begin{figure}[!htbp]
 \centering
 \includegraphics[width=0.32\textwidth]{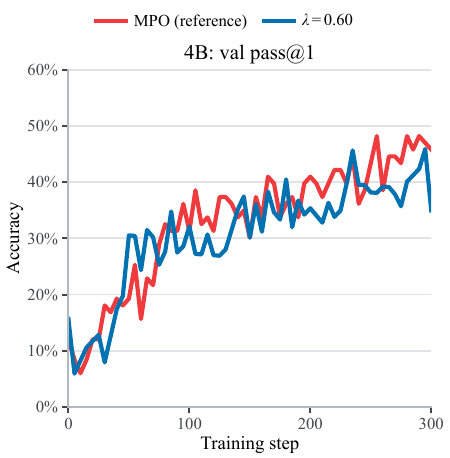}
 \hspace{0.04\textwidth}
 \includegraphics[width=0.32\textwidth]{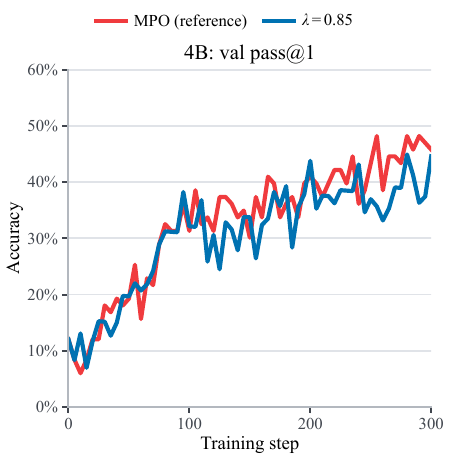}
 \caption{Cross-turn propagation $\lambda$: 0.60 (left) and 0.85 (right),
 compared with the MPO reference.}
 \label{fig:ablation-val-lambda}
\end{figure}

\begin{figure}[!htbp]
 \centering
 \includegraphics[width=0.32\textwidth]{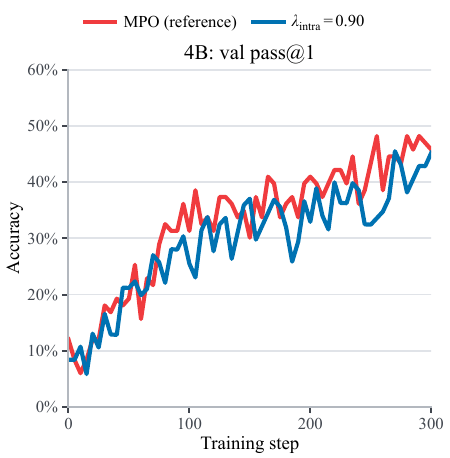}
 \hspace{0.04\textwidth}
 \includegraphics[width=0.32\textwidth]{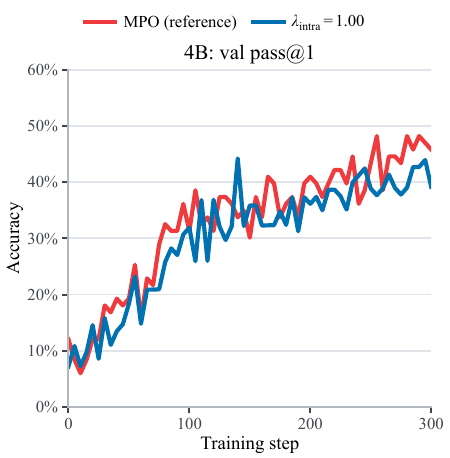}
 \caption{Within-turn trace $\lambda_{\mathrm{intra}}$: 0.90 (left) and
 1.00 (right), compared with the MPO reference.}
 \label{fig:ablation-val-intra-lambda}
\end{figure}
\FloatBarrier
\endgroup